\relax
\documentclass[a4paper]{article}

\usepackage[utf8]{inputenc}
\usepackage[T1]{fontenc}

\usepackage{times}  
\usepackage{helvet}  
\usepackage{courier}  
\usepackage[hyphens]{url}  
\usepackage{graphicx} 
\usepackage{algorithm}
\usepackage{algorithmic}

\usepackage{todonotes}

\usepackage{amsmath}
\usepackage{amssymb}
\usepackage{mathtools}
\usepackage{amsthm}

\def\showproofs{0}

\theoremstyle{plain}

\newtheorem{example}{Example}

\newtheorem{lemma}{Lemma}
\newtheorem{proposition}{Proposition}

\theoremstyle{definition}
\newtheorem{definition}{Definition}

\newcommand{\bag}{\ensuremath{\mathcal{B}}}

\newcommand{\args}{\ensuremath{\mathcal{A}}}

\newcommand{\attacker}{\ensuremath{\mathrm{Att}}}

\newcommand{\supporter}{\ensuremath{\mathrm{Sup}}}

\newcommand{\labelling}{\ensuremath{\operatorname{L}}}
\newcommand{\labellings}{\ensuremath{\mathcal{L}}}
\newcommand{\lin}{\ensuremath{\operatorname{in}}}
\newcommand{\lout}{\ensuremath{\operatorname{out}}}
\newcommand{\lundec}{\ensuremath{\operatorname{und}}}

\newcommand{\rf}{\ensuremath{\mathcal{F}}}
\newcommand{\outputf}{\ensuremath{\mathrm{O}}}
\newcommand{\dt}{\ensuremath{\mathcal{T}}}
\newcommand{\cinput}{\ensuremath{\mathbf{x}}}
\newcommand{\coutput}{\ensuremath{y}}
\newcommand{\dom}{\ensuremath{\mathcal{D}}}
\newcommand{\class}{\ensuremath{\mathcal{C}}}
\newcommand{\ambSet}{\ensuremath{\mathrm{Amb}}}

\newcommand{\prem}{\ensuremath{\mathrm{prem}}}
\newcommand{\conc}{\ensuremath{\mathrm{conc}}}

\newcommand{\bX}{\ensuremath{\mathbf{U}}}
\newcommand{\bx}{\ensuremath{\mathbf{u}}}
\newcommand{\bY}{\ensuremath{\mathbf{V}}}
\newcommand{\by}{\ensuremath{\mathbf{v}}}
\newcommand{\bZ}{\ensuremath{\mathbf{W}}}
\newcommand{\bz}{\ensuremath{\mathbf{w}}}

\newcommand{\charf}{\ensuremath{\chi_{\rf}}}

\begin{document}

\title{Explaining Random Forests using Bipolar Argumentation and Markov Networks (Technical Report)}
\author{
   Nico Potyka\\
   Imperial College London \\
  \texttt{n.potyka@imperial.ac.uk}
  \and
   Xiang Yin\\
   Imperial College London \\
  \texttt{x.yin20@imperial.ac.uk}
  \and
   Francesca Toni\\
   Imperial College London \\
  \texttt{f.toni@imperial.ac.uk}
}
\date{}

\maketitle

\begin{abstract}
Random forests are decision tree ensembles that can be used 
to solve a variety of machine learning problems. However, as
the number of trees and their individual size can be large,
their decision making process is often incomprehensible.
In order to reason about the decision process, 
we propose representing it as an 
argumentation problem.
We generalize sufficient and necessary 
argumentative explanations using a 
Markov network encoding, discuss the relevance
of these explanations and establish
relationships to families of
abductive explanations from the literature.
As the complexity of the explanation problems is high,
we discuss a probabilistic approximation algorithm and present first
experimental results.
\end{abstract}

\section{Introduction and Related Work}


Random forests (RFs) \cite{breiman2001random} are machine
learning models with various applications in areas like E-commerce, 
Finance and Medicine. They consist
of multiple decision trees that use different subsets of the available
features. Given an input, every tree makes an individual decision and
the output of the random forest is obtained by a majority vote.
They have low risk of overfitting; support both classification
and regression tasks and come equipped with some feature 
importance measures \cite{breiman2001random}. However, 
feature importance measures can be
too simplistic as they can represent neither joint effects of features (e.g., multi-drug interactions) 
nor non-monotonicity (e.g.,
increasing the weight may be healthy for an underweight person,
but not for an overweight person).

In recent years, a variety of other explanation methods has been
proposed. Model-agnostic feature importance
measures like LIME \cite{ribeiro2016should}, SHAP \cite{lundberg2017unified} and MAPLE \cite{plumb2018model}
have similar limitations like the feature
importance measures defined for random forests.
Counterfactual explanations explain how an input can be 
modified to change the decision \cite{wachter2017counterfactual},
but mainly explain the model locally. Another interesting family of explanation methods
are abductive explanations, also called prime implicant explanations
\cite{ShihCD18,Izza021,WaldchenMHK21}. Roughly speaking,
abductive explanations are sufficient reasons for a classification.
Recently, SAT encodings have been applied to compute abductive
explanations in tree ensembles \cite{Izza021,IgnatievIS022}
and many other logic-based explanation approaches
have been investigated
\cite{Silva22,Cyras0ABT21,vassiliades2021argumentation}.

As random forests are essentially composed of rules, 
a natural question is if we can use logical tools
to reason in more flexible ways about random forests. 
Since the rules can be mutually inconsistent, non-classical
reasoning approaches seem best suited.
Here, we investigate abstract bipolar argumentation graphs (BAGs)
\cite{amgoud2008bipolarity,orenN08,boella2010support,cayrol_bipolarity_2013} for this purpose.
Intuitively,
BAGs allow identifying consistent subsets (extensions) 
of contradicting arguments and to reason about them.
We will show that the bi-stable semantics for BAGs
\cite{potyka_generalizing_2021} allows representing random
forests as BAGs such that the possible decisions made by
the forest correspond to extensions of the BAG. 
Finding sufficient and necessary reasons for
the classification of a random forest can 
then be reduced to finding sufficient and
necessary reasons in argumentation frameworks \cite{BorgB21}. 
In order to solve the combinatorial reasoning problems, we consider Markov network encodings of the BAG \cite{potyka_abstract_2020}, which
also allow reasoning about almost sufficient
and almost necessary reasons. As the computational
complexity of the problems is high, we consider 
a probabilistic algorithm to approximate
reasons
and present first experimental results.

The proofs of all technical results can be found in the appendix. The sourcecode and Jupyter notebooks to reproduce the experiments are available under
\url{https://github.com/nicopotyka/Uncertainpy}
in the folders \emph{src/uncertainpy/explanation}
and \emph{examples/explanations/randomForests}, respectively.

\section{Random Forests and Classes of AXps}

We will focus on Random forests for classification 
problems here.
The goal of classification is to
assign class labels $\coutput$
to inputs $\cinput$.
Inputs are 
vectors
$\cinput = (x_1, \dots, x_k)$, where the i-th value belongs to some feature $X_i$ with domain $D_i$.
We let $\dom = \bigtimes_{i=1}^k D_i$ denote the set of all inputs and
$\class$ the set of \emph{class labels}.
Figure \ref{fig:exp_rf} shows two 
decision trees
for a medical classification problem where
patients are diagnosed based on their 
age and three symptoms 
$A, B, C$ that can be present ($1$) or not ($0$).
\begin{figure}[t!b]
	\centering
		\includegraphics[width=0.7\textwidth]{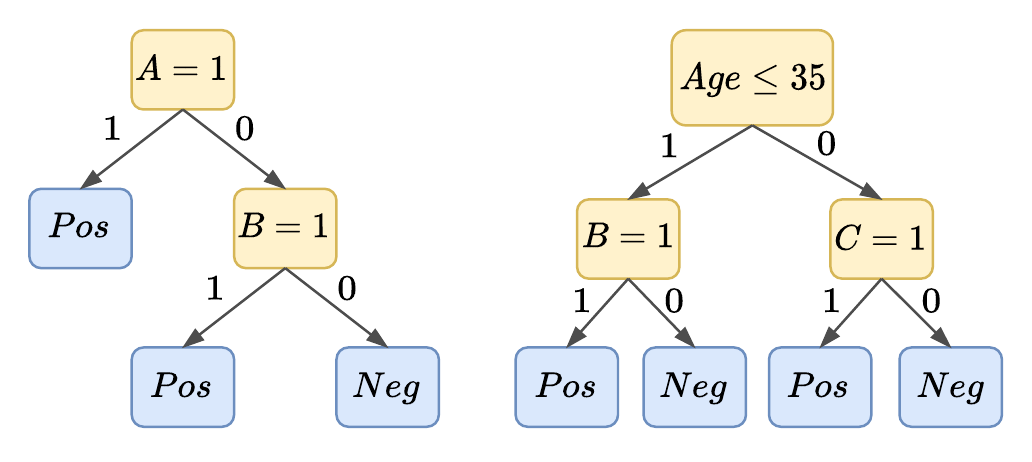}
	\caption{A simple random forest with two decision trees.}
	\label{fig:exp_rf}
\end{figure}
The diagnosis can be positive ($Pos$) or negative ($Neg$).
Formally, we understand trees as sets of
rules $\dt = \{r_1, \dots, r_{|\dt|}\}$.
A rule $r$ has the form $\prem(r) \rightarrow \conc(r)$,
where the premise $\prem(r)$ is a set of \emph{feature literals} 
and the conclusion $\conc(r) \in \class$
a class label.
Feature literals are positive or negative
\emph{feature conditions}.
Feature conditions (positive feature literals) have  
the form $X_i = v_i$ (categorical features) or 
$X_i \leq v_i$ (ordinal/numerical features), where $v_i \in D_i$.
Negative feature literals are negated feature conditions.
For example, the tree on the left in Figure
\ref{fig:exp_rf} can be represented by the three rules
$\{A=1\} \rightarrow Pos$,
 $\{A \neq 1, B=1\} \rightarrow Pos$,
$\{A \neq 1, B \neq 1\} \rightarrow Neg$.
Note that the rules are exhaustive and
exclusive, that is, for every input $\cinput \in \dom$,
there is one and only one rule that applies.
We call this rule the \emph{active rule in $\dt$ for $\cinput$}.

A random forest $\rf = \{\dt_1, \dots, \dt_t\}$ is a collection
of decision trees. It processes an input $\cinput$
by computing the outputs $\coutput_1, \dots, \coutput_t$
for $\cinput$ for all decision trees. 
Then, it returns the class label that was selected most frequently. We assume that 
$\bot$ is returned if
multiple class labels receive the maximum number of votes (a \emph{tie}).
We let 
$\outputf_{\rf}: \dom \rightarrow \class \cup \{\bot\}$
denote the \emph{output function} for $\rf$, where
$\outputf_{\rf}(\cinput) = y$ if $\rf$ outputs class $y$ for $\cinput$ and $\outputf_{\rf}(\cinput) = \bot$ if there is a tie.
\begin{example}
Consider the random forest composed of the
two trees in Figure \ref{fig:exp_rf} and
a patient aged $25$ with symptoms $A$ and $B$ present. Then both decision trees will return $Pos$
and the output of the random forest is $Pos$. If symtom $B$ was not present, the decision tree on the left would return
$Pos$, while the decision tree on the right would return $Neg$.
In this case, the output is $\bot$.
\end{example}
In order to understand the decision process of 
random forests, we can consider 
abductive explanations \cite{ShihCD18,Izza021}.
A \emph{weak abductive explanation (wAXp)}, for a class label 
$y \in \class$ is a 
partial assignment
to the features such that every 
completion $\cinput$ of this partial 
assignment satisfies 
$\outputf(\cinput) = y$
\cite{HuangIICA022}. If a wAXp cannot
be shortened, then it is called
an \emph{abductive explanation (AXp)}
or prime implicant \cite{HuangIICA022}.
For example, the partial assignment
$(A=1, B=1, Age=20)$ is a wAXp for 
$Pos$ with respect
to the random forest in Figure
\ref{fig:exp_rf}. However, it is not
an AXp because it can be shortened to $(B=1, Age=20)$,
which is, in fact,
an AXp. 
\cite{WaldchenMHK21}
recently generalized AXps.
Roughly speaking, a partial assignment is called a \emph{$\delta$-relevant explanation
for class $y$} if the probability that a
completion satisfies
$\outputf(\cinput) = y$ is at least $\delta$ 
\cite{WaldchenMHK21} (where we consider a uniform distribution over the completions). For brevity, we will call them 
$\delta$-AXps in the following. Note that $1$-AXps are wAXps. 
Finding and even deciding if a partial
assignment is an ($\delta$-)AXp is 
difficult as complexity results in
\cite{Izza021} and \cite{WaldchenMHK21}
show.

\section{Ambiguous and Indistinguishable Inputs}

Random forests may be unable
to make a decision due to a tie in the individual
tree decisions.
For binary classification problems,
we can always avoid a tie by creating a forest with an odd number of trees. 
However, if we have more than two classes, there
is no simple workaround.
We call the undecided inputs \emph{ambiguous} 
and let 
$\ambSet(\rf) = \{\cinput \in \dom \mid \outputf_{\rf}(\cinput) = \bot\}$ denote the set of all ambiguous inputs.

The following proposition explains that 
analyzing ambiguity is a difficult 
problem even for simple random 
forests that contain only boolean features
and have at most $4$ leaves/rules.
We call this special case \emph{B4L random forests}.
\begin{proposition}
\label{prop_ambiguous_complexity}
\begin{itemize}
    \item Deciding if there exists 
    an ambiguous input for a B4L random forest 
    is NP-complete.
    \item Counting the number of ambiguous inputs for a B4L random forest is \#P-complete.
\end{itemize}
\end{proposition}
\if\showproofs1
\begin{proof}
We use the terminology and conventions from \cite{papadimitriou95book} in the proof.

1. Membership follows from observing that an
ambiguous input is a polynomially-sized 
certificate that can be verified in 
polynomial time (because the random forest processes inputs in linear time).
For hardness, we show that there is a 
polynomial-time reduction from 3SAT to 
deciding the existence of ambiguous inputs
for B4L random forests. 
Given an arbitrary
3CNF formula $F: \bigwedge_{i=1}^m \big( \bigvee_{i=1}^3 l_{i,j} \big)$ with $m$ clauses, we create a decision tree of depth 3 for every clause $\big( \bigvee_{i=1}^3 l_{i,j} \big)$. It contains 
three decision nodes that contain the atoms corresponding to $l_{i,1}, l_{i,2}, l_{i,3}$. If $l_{i,j}$ is a positive literal,
the positive branch will go to a leaf labelled $1$ and the
negative branch to $l_{i,j+1}$ if $j<3$ and to a leaf labelled
$0$ otherwise. A negative literal is treated symmetrically by
switching the negative and positive branch. We illustrate the
construction in Figure \ref{fig_clause_2_tree}.
\begin{figure}[tb]
	\centering
		\includegraphics[width=0.2\textwidth]{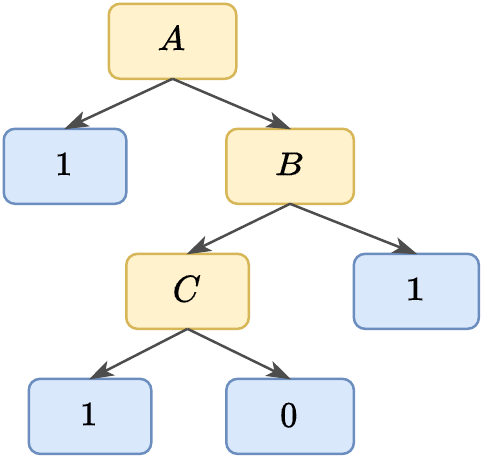}
	\caption{Tree for the 3-clause $A \wedge \neg B \wedge C$.}
	\label{fig_clause_2_tree}
\end{figure}
We now define a B4L random forest consisting of the $m$ trees
obtained from the clauses and $m$ additional tree stumps that
consists of a single leaf that is labelled $0$. Every tree
has constant size and so the size of the random forest is linear
in the size of $F$. Furthermore, 
$F$ is satisfiable 
\begin{itemize}
    \item iff all clauses can be satisfied simultaneously,
    \item iff there is an input such that all clause trees
    output $1$,
    \item iff there is an ambiguos input (because the $m$ 
    clause trees will vote $1$ and the $m$ tree stumps will vote $0$).
\end{itemize}
Hence, $F$ is satisfiable if and only
if the random forest has an ambiguous input. 

2. As (ambiguous) inputs are polynomial in the size of the 
random forest and can be processed in polynomial time, the
problem is in $\#P$. Hardness  follows from observing that
our previous reduction is a parsimonious reduction from
3SAT, that is, the number of satisfying assignments of the
3CNF formula
equals the number of ambiguous inputs of the random forest.
Hence, the claim follows from $\#P$-hardness of $\#3SAT$. 
\end{proof}
\fi
If $\rf$ contains variables with 
infinite domains,
the number of ambiguous inputs can be 
infinite. However, it is always
possible to partition the inputs into a
finite set of equivalence classes.
To make this more precise,
let us first note that every random forest
yields a natural partition of the input
domains based on the feature conditions that occur in the forest.
\begin{definition}[Domain Partition]
\label{def_dom_part}
The \emph{domain partition associated with $\rf$} partitions every domain $D_i$ into
disjoint subsets $S_{i,1}, \dots, S_{i,n_i}$
such that $D_i = \uplus_{j=1}^{n_i} S_{i,j}$.
If $D_i = \{v_1, \dots, v_{|D_i|}\}$ is finite,
then $n_i = |D_i|$ and $S_{i,j} = \{v_j\}$.
If $D_i$ is continuous,
let $X_i \leq v_1, \dots, X_i \leq v_{|D_i|}$
denote the feature conditions that occur
in $\rf$ for $X_i$ and assume w.l.o.g. that $v_1 \leq \dots \leq v_{|D_i|}$.
Then $n_i = {|D_i|} +1$, 
$S_{i,j} = (v_{j-1}, v_j] = 
\{v \in D_i \mid v_{j-1} < v \leq v_j \}$
where $v_0 = \inf D_i$,
and 
$S_{|D_i|+1} = [v_{|D_i|} ,\sup D_i)$.
\end{definition}
\begin{example}
For the random forest in Figure \ref{fig:exp_rf}, the domains of $A$, $B$ 
and $C$ are partitioned into $\{0\}$
and $\{1\}$. For $Age$, the domain
is partitioned into $(- \infty, 35]$
and $(35, \infty)$.
\end{example}
Note that the number of partitioning
sets is always finite because random forests
are finite. Furthermore, when we chose one
partition index $i_j$ for every feature $X_i$,
then all inputs in 
$S_{i_1} \times \dots \times S_{i_k} \subseteq \dom$
are indistinguishable for the trees 
and, therefore, are all classified in
the same way.
To capture these indistinguishable inputs,
we define the \emph{characteristic function
of $\rf$} as the mapping 
$\charf: \dom \rightarrow \mathbb{N}^k$
that maps every input $\cinput$ to a 
$k$-dimensional vector $v = \charf(\cinput)$
such that
$x_i \in S_{i, v_i}$ for
all $i = 1, \dots, k$.
\begin{definition}[Indistinguishability Relation]
\label{def_indisting_rel}
Two inputs $\cinput_1, \cinput_2 \in \dom$
are \emph{indistinguishable with respect to 
$\rf$} iff $\charf(\cinput_1) = \charf(\cinput_2)$. We denote this by $\cinput_1 \equiv_\rf \cinput_2$.
\end{definition}
It is easy to check
that indistinguishability is an equivalence 
relation and that the equivalence classes 
$E \in \dom /\! \equiv_\rf$
correspond to the sets
$S_{i_1} \times \dots \times S_{i_k}$
that we obtain by choosing one partition
index $i_j$ for every feature. 

Let us note that while
$\ambSet(\rf)$ can be infinite,
the set of equivalence classes of
ambiguous inputs
$\ambSet(\rf) /\! \equiv_\rf$ is always
finite. Hence,
we can now ask, what is
the number of ambiguous equivalence classes? If all domains are
finite, this is equivalent to counting
the number of ambiguous inputs because,
in this case, every equivalence class 
contains exactly one input. 
Hence, Proposition
\ref{prop_ambiguous_complexity}
implies that counting the ambiguous equivalence classes
is $\#P$-hard as well.

\section{Representing Random Forests as BAGs}

In order to reason about the decision process
of random forests, we represent it as a 
\emph{bipolar argumentation graph (BAG)}.
Formally, a BAG is a tuple
$\bag = (\args, \attacker, \supporter)$, where $\args$ is a finite set of arguments, $\attacker \subseteq \args \times \args$ is the \emph{attack relation} and 
$\supporter \subseteq \args \times \args$ is the support relation \cite{cayrol_bipolarity_2013}. 
We let $\attacker(A) = \{B \mid (B,A) \in \attacker\}$ denote the attackers of $A$ and, analogically, 
$\supporter(A)$ its supporters.

Various semantics have been proposed for BAGs. 
We use the bi-complete semantics from \cite{potyka_generalizing_2021}
here, which generalizes the complete
semantics \cite{dung_acceptability_1995}
and resolves conflicts between attackers
and supporters by means of majority votes. 
It is based on labellings $L: \args \rightarrow \{\lin, \lout, \lundec\}$
that assign a label $\lin$ (accept), $\lout$ (reject) or
$\lundec$ (undecided) to every argument.
Given a labelling $\labelling$, we say that the attackers of an argument \emph{dominate} its supporters if 
$|\{B \in \attacker(A) \mid \labelling(B) = \lin\}| > |\{B \in \supporter(A) \mid \labelling(B) \neq \lout\}|$.
That is, for every supporter that is not out, there is an attacker that is in and there is at least one additional
attacker that is in. Intuitively, every non-rejected pro-argument is balanced out by an accepted counterargument
and there is an additional counterargument that breaks a potential tie.
Symmetrically, the supporters of an argument \emph{dominate} its attackers if
$|\{B \in \supporter(A) \mid \labelling(B) = \lin\}| > |\{B \in \attacker(A) \mid \labelling(B) \neq \lout\}|$.
Given a BAF $(\args, \attacker, \supporter)$, we call a labelling $\labelling: \args \rightarrow \{\lin, \lout, \lundec\}$
\begin{description}
	\item[Bi-complete \cite{potyka_generalizing_2021}: ] if $\labelling$ satisfies
		\begin{enumerate}
				\item $\labelling(A) = \lin$ if and only if $\labelling(B) = \lout$ for all $B \in \attacker(A)$ or 
				 $A$'s supporters dominate its attackers.
				\item $\labelling(A) = \lout$ if and only if $A$'s attackers dominate its supporters.
		\end{enumerate}
\end{description}
A bi-complete labelling is called \emph{bi-stable}
if it does not label any argument undecided.
We let $\labellings^c(\bag)$ and $\labellings^s(\bag)$ 
denote the bi-complete and bi-stable labellings of
the BAG $\bag$.

Given a random forest $\rf = \{\dt_1, \dots, \dt_t\}$,
we want to represent it as a BAG $\bag_{\rf,\cinput}$ such that the labellings of $\bag_{\rf,\cinput}$
correspond to the possible inputs and decisions
of $\rf$.
To do so, we first associate a collection
of arguments with $\rf$.
\begin{definition}[Explanation Arguments]
The \emph{explanation arguments} 
$\args_{\rf} = \args_{\class} \cup \args_R \cup \args_F$ associated with the random forest $\rf$ are defined as follows:
\begin{itemize}
    \item $\args_\class = \{A_y \mid y \in \class\}$ contains one \emph{class arguments} for every class,
    \item  $\args_{R} = \bigcup_{\dt \in \rf} \args_{\dt}$, where 
    $\args_{\dt} = \{A_{\dt, r} \mid r \in \dt\}$ contains a \emph{rule argument} $A_{\dt, r}$ for every tree $\dt$ in $\rf$ and every rule $r \in \dt$,
    \item $\args_F = \bigcup_{i=1}^n \args_{X_i}$, contains one 
    \emph{feature argument}
    for every partitioning set of the 
    feature domain (Def. \ref{def_dom_part}),
    that is,
    $\args_{X_i} = \{A_{X_i \in S_{i,1}}, \dots, 
    A_{X_i \in S_{i,n_i}}\}$.
    \end{itemize}
\end{definition}

Next, we explain the attack and support relations in $\bag_{\rf}$. Intuitively, $\bag_{\rf}$ is a layered
graph with the feature arguments $\args_F$
at the bottom, the rule arguments $\args_R$ in the middle
and the class arguments $\args_\class$ at the top.
Attack edges occur only within the feature layer, 
from the feature to the rule and 
from the rule to the class layer.
Support edges occur only 
from the rule to the class layer.
\begin{definition}[Explanation Argument Relationships]
The attack and support relationships $\attacker_{\rf} = \attacker_{F,F} \cup \attacker_{F,R} \cup \attacker_{R, \class}$ and $\supporter_{\rf} =  \supporter_{R, \class}$ associated with 
the random forest $\rf$ are defined as 
follows:
\begin{itemize}
    \item $\attacker_{F,F}$ 
    contains a \emph{feature-feature-attack} 
    between all feature arguments that belong to the same feature. That is, 
    $(A_{f_1}, A_{f_2}) \in \attacker_{F,F}$
    if and only if 
    $A_{f_1}, A_{f_2} \in \args_{X_i}$,
    \item  $\attacker_{F,R}$ contains a \emph{feature-rule-attack}
    $(A_{X_i \in S_{i,j}}, A_{\dt, r})$ if
    the feature constraint $X_i \in S_{i,j}$ is inconsistent 
    with a feature literal 
    $L \in \prem(r)$ (e.g., the feature constraint $X \in (1,3]$ is inconsistent with
    the feature literal $X>6$),
    \item $\attacker_{R,\class}$ contains 
    a \emph{rule-class-attack}
    $(A_{\dt, r}, A_y)$ for every rule argument
    $A_{\dt, r}$ with $\conc(r) \neq y$,
    \item  $\supporter_{R,\class}$ contains a \emph{rule-class-support}  
    $(A_{\dt, r}, A_y)$ for every rule argument
    $A_{\dt, r}$ with $\conc(r) = y$.
\end{itemize}
\end{definition}
Intuitively, feature-feature attacks
guarantee that only one feature argument
per feature can be accepted (because they refer to distinct feature values/ranges).
Feature-rule attacks deactivate rules that
are inconsistent with the currently
accepted feature configuration.
The rule-class relationships support/attack
classes according to their claim.
The Explanation BAG associated with
a random forest is then constructed from
the explanation arguments and the attack
and support relationships between them.
\begin{definition}[Explanation BAG]
Given a random forest $\rf$, the \emph{explanation BAG} for $\rf$ is the BAG $\bag_{\rf}= (\args_{\rf}, \attacker_{\rf}, \supporter_{\rf})$.
\end{definition}
We note that $\bag_{\rf}$ can be constructed in quadratic time. The reason for the quadratic
blowup is that we have pairwise attacks 
between feature arguments for the same
feature.
\begin{proposition}
$\bag_{\rf}$ can be generated from $\rf$ in quadratic time.
\end{proposition}
\if\showproofs1
\begin{proof}
We have one class argument per class,
one rule argument per leaf in a decision tree and one feature argument per feature condition. Hence, the number of arguments is
linear in the size of the random forest.
Since the number of edges between the arguments
can be at most quadratic, the claim follows.
\end{proof}
\fi

\subsection{Faithfulness of the Explanation BAG}

As we show next, the explanation BAG $\bag_{\rf}$
is a faithful representation of $\rf$ in the
following sense: 
every bi-stable labelling
of $\bag_{\rf}$ 
represents a possible decision made by 
$\rf$ (\emph{correctness}) and for every possible decision that $\rf$ can make, 
there is a bi-stable labelling of $\bag_{\rf}$
that represents it (\emph{completeness}).
The following lemma motivates the use of bi-stable labellings.
\begin{lemma}
\label{lemma_bistable_feature_properties}
Let $L$ be a labelling for $\bag_{\rf}$.
\begin{enumerate}
    \item If $L \in \labellings^c(\bag_{\rf})$, then for all features $X_i$, either
    \begin{itemize}
    \item $L(A_{X_i \in S_{i,j}}) = \lundec$ for all $A_{X_i \in S_{i,j}} \in \args_{X_i}$ or
    \item $L(A_{X_i \in S_{i,j}}) = \lin$ for exactly one $A_{X_i \in S_{i,j}} \in \args_{X_i}$ and $L(A_{X_i \in S_{i,j'}}) = \lout$ for all
    other $A_{X_i \in S_{i,j'}} \in \args_{X_i} \setminus \{A_{X_i \in S_{i,j}}\}$. 
    \end{itemize}
    \item If $L \in \labellings^s(\bag_{\rf})$, then for all features $X_i$,
    $L(A_{X_i \in S_{i,j}}) = \lin$ for exactly one $A_{X_i \in S_{i,j}} \in \args_{X_i}$ and $L(A_{X_i \in S_{i,j'}}) = \lout$ for all
    other $A_{X_i \in S_{i,j'}} \in \args_{X_i} \setminus \{A_{X_i \in S_{i,j}}\}$.
\end{enumerate}
\end{lemma}
\if\showproofs1
\begin{proof}
1. Note that, by construction, all arguments in $\args_{X_i}$
 attack each other.
 Furthermore, feature arguments do not have supporters.
 Hence, if one argument in $\args_{X_i}$ is $\lundec$,
 then all others must be $\lundec$ as well.
 Since arguments in $\args_{X_i}$ are only attacked by arguments in $\args_{X_i}$,
 if one argument in $\args_{X_i}$ is $\lout$, 
 then there must be another argument in $\args_{X_i}$
 that attacks it and is $\lin$. Hence, we are either in
 the situation where all arguments are undecided or 
 at least one argument is $\lin$.
 However, if $L(A_{X_i \in S_{i,j}}) = \lin$, then all other $A_{X_i \in S_{i,j'}} \in \args_{X_i}$ must be out.

2. The claim follows immediately from item 1 because
 bi-stable labellings do not allow labelling arguments
 undecided.
\end{proof}
\fi
Lemma \ref{lemma_bistable_feature_properties}
states that bi-complete labellings either accept exactly
one feature constraint per feature or remain undecided. Since the undecided case is not 
interesting for our purposes, we focus on bi-stable
labellings.
The fact that bi-stable labellings
accept exactly one constraint per feature
allows us to associate every 
bi-stable labelling $L$ of $\bag_{\rf}$ with 
an equivalence class
$S_L \in \dom /\! \equiv_\rf$
of inputs
with respect to the indistinguishability
relation (Def. \ref{def_indisting_rel}). 

As we show next, every bi-stable labelling 
$L$ accepts exactly one
rule argument per tree. 
This rule argument corresponds to the
active path in the tree for all inputs $\cinput \in S_L$ in the corresponding equivalence class
$S_L \in \dom /\! \equiv_\rf$.
\begin{lemma}
\label{lemma_bistable_rule_properties}
If $L \in \labellings^s(\bag_{\rf})$, then for all trees $\dt \in \rf$,
$A_{\dt, r} \in \args_{\dt}$ is labelled $\lin$
if and only if $r$ is the active rule in $\dt$ 
for all inputs $\cinput \in S_L$. Furthermore, all other rule
arguments in $\args_{\dt}$ are labelled out.
\end{lemma}
\if\showproofs1
\begin{proof}
Every input
$\cinput \in S_L$ satisfies all
feature constraints accepted by $L$. 
Therefore, a rule argument $A_{\dt, r} \in \args_{\dt}$ 
is supported (attacked) by
an $\lin$-labelled feature argument if and only if
$\cinput$ satisfies (violates) the corresponding feature
condition. Therefore, $A_{\dt, r} \in \args_{\dt}$ is labelled $\lin$ if and only if $r$ is the active rule in $\dt$ for $\cinput$

The second claim follows from the same considerations
because the inactive rules must be attacked by the
$\lin$-labelled arguments.
\end{proof}
\fi
We can now show that our
enconding is correct in the sense that
a class argument $A_y$ can be labelled 
$\lin$ by $L$ if and only if $\outputf_{\rf}(\cinput) = y$
for all $\cinput \in S_L$.
\begin{proposition}[Correctness]
\label{prop_bistable_class_properties}
If $L \in \labellings^s(\bag_{\rf})$, then for all class arguments
$A_y \in \args_{\class}$,
$L(A_y) = \lin$ if and only if $\outputf_{\rf}(\cinput) = y$ for all $\cinput \in S_L$. Furthermore, if $L(A_y) = \lin$, then
$L(A_{y'}) = \lout$ for all $A_{y'} \in \args_{\class}$.
\end{proposition}
\if\showproofs1
\begin{proof}
If $L(A_y) = \lin$, then all attacking arguments must
be $\lout$ or the supporters must dominate the attackers.
If all attackers are $\lout$, then no rule that picks
another class can be active in $\dt$ and so all active
rules must pick $y$. If the supporters dominate the 
attackers, then a majority of active rules (trees) 
vote for $y$.
Hence, in both cases $\outputf_{\rf}(\cinput) = y$
for all $\cinput \in S_L$.

Conversely, if $\outputf_{\rf}(\cinput) = y$
for all $\cinput \in S_L$,
then a majority of trees vote for $y$. The corresponding
rules will then support $A_y$ and dominate the attackers,
so that $L(A_y) = \lin$.

For the second claim, notice that 
$L(A_y) = \lin$ implies that a majority of $\lin$
rule features support $A_y$. However, by construction,
these rule features also attack all other class arguments.
Therefore, for these arguments, the attackers must
dominate the supporters, so that they will be labelled out.
\end{proof}
\fi
Proposition \ref{prop_bistable_class_properties} guarantees that
every bi-stable labelling $L$ represents a collection of 
inputs from the equivalence class
$S_L \in \dom /\! \equiv_\rf$
and the accepted class-arguments corresponds to their classification.
However, it is also possible that $L$ does not
accept any class argument.
As we show next, this is only possible
if the inputs in $S_L \in \dom /\! \equiv_\rf$ are ambiguous. Since all
inputs in an equivalence class are classified
equally, this is equivalent to showing
that for every input with $\outputf_{\rf}(\cinput) \neq \bot$,
there is a corresponding bi-stable
labelling $L_{\cinput}$ that represents it. 
$L_{\cinput}$ is defined as follows:
\begin{enumerate}
    \item a feature argument  $A_{X_i \in S_{i,j}} \in \args_F$ is labelled $\lin$
    if $X_i \in S_{i,j}$ and labelled $\lout$
    otherwise,
    \item a rule argument $A_{\dt, r} \in \args_R$ is labelled 
    $\lin$ if $r$ is the active rule in $\dt$ for $\cinput$ and labelled $\lout$ otherwise,
    \item a class argument is labelled $\lin$ if its supporters dominate its attackers, $\lout$ if its attackers dominate its supporters, and $\lundec$
    otherwise.
\end{enumerate}
The following lemma explains that
$L_{\cinput}$ is always a bi-complete labelling (Item 1) and  accepts at most one class argument (Item 2).

\begin{lemma}
\label{lemma_bistable_input_properties}
\begin{enumerate}
    \item For all $\cinput \in \dom$, $L_{\cinput} \in \labellings^c(\bag_{\rf})$ .
    \item There is at most one $y \in \class$ such that
    $L_{\cinput}(A_y) = \lin$. Furthermore, if $L_{\cinput}(A_y) = \lin$ for some $y \in \class$, then $L_{\cinput}(A_{y'}) = \lout$ for all $y' \in \class \setminus \{y\}$.
\end{enumerate}
\end{lemma}
\if\showproofs1
\begin{proof}
1. We start by checking that the labelling of the 
feature arguments satisfies the definition of a
bi-complete labelling.
For every feature $X_i$, exactly one feature argument
$A_{X_i \in S_{i,j}} \in \args_{X_i}$ is $\lin$ because the feature conditions are mutually exclusive by definition. 
Since $A_{X_i \in S_{i,j}}$ is only attacked by other feature arguments
in $\args_{X_i}$ that are out, it must indeed be $\lin$. 
Since all other feature arguments in $\args_{X_i}$ 
are attacked by $A_{X_i \in S_{i,j}}$ they must indeed be $\lout$. 
Hence, the feature arguments are correctly labelled.

The fact that all rule arguments are correctly labelled
follows by the same arguments that we used in the proof 
of Lemma \ref{lemma_bistable_rule_properties}.

Finally consider a class argument $A_y$. 
Using the same arguments as in the proof of 
Proposition \ref{prop_bistable_class_properties},
we can check that $A_y$ is labelled $\lin$ ($\lout$)
if and only if its supporters (attackers)
dominate its supporters. 
If $L_{\cinput}(A_y) = \lundec$, the number of $\lin$-attackers equals the number of $\lin$-supporters.
The definition of bi-complete labellings could only be
violated if this number is zero. However, since every tree
must contain an active rule, the number must be non-zero.
Hence, $L_{\cinput}$ is a bi-complete labelling.

2. Suppose that $L_{\cinput}(A_y) = \lin$. Then the
number of rule arguments (active rules/ trees) that are
$\lin$ and support $A_y$ must be larger than the number 
of rule arguments that are $\lin$ and attack $A_y$. Hence,
for every other class label $y' \in \class \setminus \{y\}$,
the number of rule arguments that are
$\lin$ and attack $A_{y'}$ must be larger than the number 
of rule arguments that are $\lin$ and support $A_{y'}$.
Hence, $L_{\cinput}(A_{y'}) = \lout$.
\end{proof}
\fi
We can now show that our
encoding is complete in the sense that
$L_{\cinput}$ is a bi-stable labelling
($\cinput$ is represented by a bi-stable labelling)
if and only if $\outputf_{\rf}(\cinput) \neq \bot$. 
\begin{proposition}[Completeness]
\label{prop_arg_encoding_completeness}
For all inputs $\cinput \in \dom$,
$\outputf_{\rf}(\cinput) \neq \bot$
if and only if
$L_{\cinput} \in \labellings^s(\bag_{\rf})$.
\end{proposition}
\if\showproofs1
\begin{proof}
Item 1 of Lemma \ref{lemma_bistable_input_properties}
guarantees that $L_{\cinput}$ is bi-complete.
By construction of $L_{\cinput}$, 
no feature and
rule arguments are labelled undecided.
Hence, the 
only case in which $L_{\cinput}$ is not a
bi-stable labelling is the case in which 
it labels a class argument undecided. 

If $L_{\cinput}$ is
bi-stable, there must be a class label $y \in \class$ such that the
number of rule arguments (active rules/ trees) that are
$\lin$ and support $A_y$ is larger than the number 
of rule arguments that are $\lin$ and attack $A_y$.
Hence, a majority of trees support $y$ and 
$\outputf_{\rf}(\cinput) = y \neq \bot$.
Conversely, if $\outputf_{\rf}(\cinput) = y \neq \bot$,
a majority of trees support $y$ and the corresponding 
active rule arguments that support $A_y$ will dominate 
$A_y$'s attackers. Hence, $L_{\cinput}(A_y) =\lin$ and
according to item 2 of Lemma \ref{lemma_bistable_input_properties}, $L_{\cinput}(A_{y'}) =\lout$
for the remaining class arguments. Hence, 
$L_{\cinput}$ is bi-stable.
\end{proof}
\fi

\subsection{Applications of the Explanation BAG}

Now that we established the formal relationship
between $\bag_{\rf}$ and  $\rf$, we can
use it to reduce questions about
$\rf$ to argumentation problems in $\bag_{\rf}$.
To begin with, we note that counting
the ambiguous equivalence classes of $\rf$
can be reduced to counting the bi-stable 
labellings of $\bag_{\rf}$.
\begin{proposition}
\label{prop_count_ambigous_vs_labellings}
$|\ambSet(\rf) /\! \equiv_\rf| = |\dom| - |\labellings^s(\bag_{\rf})|$.
\end{proposition}
\if\showproofs1
\begin{proof}
By correctness and completeness of our
encoding, 
$|\labellings^s(\bag_{\rf})|$ is the number
of non-ambiguous equivalence classes, 
that is, $|\ambSet(\rf) /\! \equiv_\rf| + |\labellings^s(\bag_{\rf})| = |\dom|$.
Reordering terms gives the claim.
\end{proof}
\fi
Two interesting argumentative reasoning 
problems that are relevant for explainable AI 
are
finding sufficient and necessary reasons for
the acceptance of arguments \cite{BorgB21}.
A set of arguments
$S$ is a \emph{sufficient reason}
for an argument $A$ if for all
labellings $L$, $A$ is accepted by $L$
whenever $S$ is accepted by $L$.
A set of arguments $N$ is a 
\emph{necessary reason} for $A$ if
$L$ accepts $A$ only if it also accepts
$N$. We will consider sufficient and
necessary reasons with respect to bi-stable labellings here.
Note that a set of feature
arguments
$\{A_{X_{i_1} \in S_{{i_1},j_1}},
\dots,
A_{X_{i_k} \in S_{{i_k},j_k}}\}$
is a (minimal) sufficient
reason for a class argument $A_y$ in $\bag_{\rf}$
if and only if every partial assignment
from 
$S_{{i_1},j_1} \times
\dots \times S_{{i_k},j_k}$
is a wAXp (AXp) for $y$ in $\rf$. 
\begin{example}
For the explanation BAG corresponding to
Figure \ref{fig:exp_rf},
the set of feature arguments 
$\{A_{B \in \{1\}}, A_{Age \in (-\infty,35]}\}$
is a minimal sufficient reason for 
$A_{Pos}$. This means that every 
partial assignment of the form
$(B=1, Age=x)$, where $x\leq 35$, is an AXp
for the random forest.
\end{example}
Similarly, if 
$\{A_{X_{i_1} \in S_{{i_1},j_1}},
\dots,
A_{X_{i_k} \in S_{{i_k},j_k}}\}$
is a necessary reason for $A_y$,
then $\rf$ can only classify an input
as $y$ if the input is an extension of
one of the partial assignments from 
$S_{{i_1},j_1} \times
\dots \times S_{{i_k},j_k}$. 
\begin{example}
For Figure \ref{fig:exp_rf},  
the feature argument $A_{A \in \{0\}}$ is necessary for $A_{Neg}$ because if $A=1$,
the first tree will vote for $Pos$, so
that the output of $\rf$ is either $Pos$ or $\bot$.
\end{example}
The following proposition allows us to
construct necessary feature arguments bottom-up.
\begin{proposition}
\label{prop_nec_reason_bottom_up}
If $N \subseteq \args_{\rf}$ is necessary for
$A_y$, then all 
$A \in N$ are necessary for 
$A_y$.
\end{proposition}
\if\showproofs1
\begin{proof}
Consider an arbitrary labelling $L$ that accepts
$A_y$. Then $L$ also accepts $N$ because $N$
is necessary for $A_y$. Since $A \in N$,
$L$ also accepts $A$. Hence, $A$ is also
necessary for accepting $A_y$.
\end{proof}
\fi
This suggests the following algorithm for 
finding all necessary feature arguments.
For every class argument $A_y$ and every 
feature argument $A_{X_i \in S_{i,j}}$, test 
if $A_{X_i \in S_{i,j}}$ is necessary for $A_y$.
The union of all these feature arguments is 
then the maximal necessary reason among the
feature arguments and we can find it with a linear number of atomic necessity checks. However, deciding 
if a feature argument is necessary for a 
class argument, may be a
difficult problem itself. The problem is in
$CoNP$ because a counterexample for the
necessity of a candidate can be verified efficiently, but
we currently do not know a lower bound for the
complexity.

\section{Markov Network Representation}

We can reduce many combinatorial tasks in argumentation graphs 
to probabilistic queries in Markov networks \cite{potyka_abstract_2020}. 
The reduction 
also allows us to generalize the idea
of necessary and sufficient reasons to
$\delta$-sufficient and $\delta$-necessary
reasons similar to the idea of $\delta$-AXps.

Intuitively, Markov networks
decompose a large probabilistic 
model $P$ into smaller local models
\cite{koller_probabilistic_2009}.
We denote random variables by capital letters $U, V, W$ and
values of these random variables by small letters $u, v, w$.
Bold capital letters $\bX, \bY, \bZ$ denote ordered sequences of random variables
and bold small letters $\bx, \by, \bz$ denote assignments to these random variables.
For example, if $\bX = (U_1, U_2, U_3)$ and $\bx = (u_1, u_2, u_3)$, then $\bX = \bx$ denotes the assignment $(U_1 = u_1, U_2 = u_2, U_3 = u_3)$.
We write $\bY \subseteq \bX$ if the random variables in $\bY$ form a subset of the random variables in $\bX$.
If $\bY \subseteq \bX$, we denote by $\bX|_{\bY}$ and $\bx|_{\bY}$ the restriction of $\bX$
and $\bx$ to the random variables in $\bY$.
For example, if $\bY = (U_1, U_3)$, we have
$\bX|_{\bY} = (U_1, U_3)$ and $\bx|_{\bY} = (u_1, u_3)$. We 
consider three types of random variables
in our application.
\begin{definition}[Explanation Random Variables] The random variables associated with $\rf$ are defined as follows:
\begin{itemize}
    \item For every feature $X_i$, we introduce a \emph{feature variable}  $U_i$ that can take values from  
    $\{S_{i,j_1}, \dots, S_{i,n_i}\}$
    (the partitioning sets of the 
    feature domain from Def. \ref{def_dom_part}).
\item For every tree $\dt = \{r_1, \dots, r_k\}$, we introduce a \emph{tree variable} 
    $U_{\dt}$ that can take values from $\{r_1, \dots, r_k\}$.
\item We introduce a \emph{class variable} 
$U_{\class}$ that can take values from $\class$.
\end{itemize}
\end{definition}
A \emph{factor} with scope $\bY \subseteq \bX$ is a function $\phi(\bY)$ that maps every assignment $\by$
to $\bY$ to a non-negative real number.  Intuitively, factors can increase or
decrease the probability of variable assignments.
Given a set of factors $\Phi = \{\phi_1(\bX_1), \dots, \phi_k(\bX_k)\}$, $\bX_i \subseteq \bX$, we define the 
\emph{plausibility of a state} of $\bX$ via
$$\mathrm{Pl}_\Phi(\bX) = \prod_{i=1}^k \phi_i(\bX|_{\bX_i}).$$
By normalizing the plausibility, we obtain a probability 
distribution that is called
the \emph{Gibbs distribution} over $\bX$:
$$P_\Phi(\bX) = \frac{1}{Z} \mathrm{Pl}_\Phi(\bX),$$
where the normalization constant 
$Z = \sum_{\bx} \mathrm{Pl}_\Phi(\bx)$
guarantees that the probabilities add up to $1$. 
$Z$ is also called the \emph{partition function}.

In our application, we build up the Gibbs distribution from
two types of factors. Intuitively,
the first one simulates the individual
tree decisions based on the state of
the feature constraints and the second 
one simulates the decision making process
of the random forest based on the tree
decisions.
\begin{definition}[Explanation Factors]
The factors associated with $\rf$ are defined as follows:
\end{definition}
\begin{itemize}
    \item For every tree $\dt \in \rf$, there is
    a tree factor $\phi_{\dt}(\bX_{\dt})$,
    where $\bX_{\dt}$ contains the
    tree variable $U_{\dt}$ and for
    each feature $X_i$ used in $\dt$,
    the corresponding feature variable $U_i$. 
    $\phi_{\dt}(\bX_{\dt})$ is a tree-factor \cite{koller_probabilistic_2009}
    defined as follows:
    given a variable assignment $\bx_{\dt}$,
    $\phi_{\dt}(\bX_{\dt})$ computes the 
    active rule
    $r$ for the assignment of the feature
    variables and returns $1$ if $r$ is
    assigned to $U_{\dt}$ and $0$
    otherwise.
\item There is one class factor
$\phi_{\class}(\bX_{\class})$, where 
$\bX_{\class}$ contains the class 
variable $U_{\class}$ and all tree variables. $\phi_{\class}(\bX_{\class})$ is defined as a 
deterministic factor \cite{koller_probabilistic_2009}
defined as follows:
Given a variable assignment $\bx_{\class}$,
$\phi_{\class}(\bX_{\class})$ iterates over 
the tree variables
and counts for every class the number
of rules that vote for the class.
It then returns $1$ if the class 
assigned to $U_{\class}$ has a larger
number of votes than all other classes
and $0$ otherwise.
\end{itemize}
The factors define the explanation plausibility distribution
and the corresponding Gibbs distribution for $\rf$.
\begin{definition}
Given a random forest $\rf$, 
the associated \emph{explanation plausibility distribution} for $\rf$ is 
$$\mathrm{Pl}_{\rf}(\bX) =
\phi_\class(\bX|_{\bX_{\class}}) \cdot
\prod_{\dt \in \rf} \phi_\dt(\bX|_{\bX_{\dt}})$$
and the \emph{explanation Gibbs distribution}
is 
$$P_{\rf}(\bX) = \frac{1}{Z} \mathrm{Pl}_{\rf}(\bX).$$
\end{definition}
Although $P_{\rf}(\bX)$ is motivated by the explanation BAG,
we can construct it immediately from $\rf$. To do this,
we traverse all trees to create the domains
of the random variables, translate the decision trees into 
tree factors and create the class factor. This can almost be done in linear time, but as we 
need to order the threshold values of continuous
features for the domain partition, there can be a
log-linear blowup.
As the explanation plausibility distribution is just the 
product of the factors, we can generate it in log-linear time.
\begin{proposition}
The explanation plausibility distribution 
$\mathrm{Pl}_{\rf}(\bX)$
can be generated from $\rf$ in log-linear time.
\end{proposition}
Building up the Gibbs distribution probably requires
exponential time as it involves computing the normalization
constant $Z$. However, we will exploit 
the fact that the plausibility distribution can be used to design
sampling algorithms to approximate $Z$ and queries to the Gibbs distribution.

\subsection{Explanation Queries}

Before going into the sampling algorithms,
let us explain what we can learn from the normalization
constant and probabilities from the Gibbs distribution.
We keep exploiting the fact that bi-stable labellings
correspond to non-ambiguous inputs for $\rf$
(Proposition \ref{prop_arg_encoding_completeness}).
To do so, we associate every input $\bx$ for $P_\rf(\bX)$
with a labelling $L_{\bx}$ as follows:
\begin{enumerate}
    \item a feature argument $A_{X_i \in S^i_j} \in \args_F$ is labelled $\lin$ if $U_i = S^i_j$ and labelled $\lout$
    otherwise,
    \item a rule argument $A_{\dt, r} \in \args_R$ is labelled 
    $\lin$ if $U_\dt = r$ and labelled $\lout$ otherwise,
    \item a class argument $A_y$ is labelled $\lin$ if 
    $U_\class = y$ and labelled $\lout$ otherwise.
\end{enumerate}
Let us first observe that the plausibility of every input for  $P_\rf(\bX)$ 
is either $0$ or $1$ and it is non-zero if and only if it represents a bi-stable labelling
of the explanation BAG.
\begin{proposition}
\label{prop_plausibility}
For every assignment $\bx$ to $\mathrm{Pl}_{\rf}(\bX)$,
we have $\mathrm{Pl}_{\rf}(\bx) \in \{0,1\}$.
Furthermore, $\mathrm{Pl}_{\rf}(\bx) \neq 0$ if and only if 
$L_{\bx}$ is a bi-stable labelling of the explanation BAG.
\end{proposition}
\if\showproofs1
\begin{proof}
For the first claim, observe that every factor returns either $0$ or $1$.
Since $\mathrm{Pl}_{\rf}(\bx)$ is the product of the returned numbers, it must
be either $0$ or $1$.

For the second claim, first assume that $\mathrm{Pl}_{\rf}(\bx) \neq 0$.
Note that, by definition, $L_{\bx}$ accepts exactly 
one feature argument per feature, one rule argument per tree
and one class argument. Using the same line of reasoning as in item 1 of the proof of Lemma 
\ref{lemma_bistable_input_properties},
we can check that $L_{\bx}$ is bi-stable.

Now assume that $\mathrm{Pl}_{\rf}(\bx) = 0$.
This can only happen if one factor returns $0$.
If a tree factor $\phi_{\dt}(\bX_{\dt})$ returns $0$, 
then $U_{\dt} \neq r$, where $r$ is the active rule
for the assignment of the feature variables. 
Hence, $L_{\bx}(A_{\dt,r}) = \lout$.
However, since $r$ is the active rule, all feature variables
are assigned feature constraints that are consistent with
the rule. Hence, the feature arguments that are in are all
consistent with $A_{\dt,r}$. Hence, all attackers of $A_{\dt,r}$
and must be labelled $\lin$ by a bi-stable labelling.
Hence, $L_{\bx}$ is not a bi-stable labelling.

Similarly, if the class factor $\phi_{\class}(\bX_{\class})$
returns $0$, then the class $y$ assigned to $U_\class$ does 
not win the majority vote. By definition of
$L_{\bx}$, we have $L_{\bx}(A_y) = \lin$, but $A_y$'s supporters
do not dominate its attackers (for otherwise, $y$ would have 
won the majority vote). Hence, $L_{\bx}$ is again not a 
bi-stable labelling.
\end{proof}
\fi
This relationship allows us to
connect the partition function $Z$
to the number of bi-stable labellings
(non-ambiguous inputs) and probabilistic queries to
generalizations  of 
sufficient and necessary reasons.
We say that a set of arguments
$S$ is a $\delta$-\emph{sufficient reason}
for an argument $A$ if among the labellings 
that accept $S$, $\delta \cdot 100$ \%
also accept $A$.
Similarly, $N$ is a 
$\delta$-\emph{necessary reason} for $A$ if
among the labellings that accept $A$, $\delta \cdot 100$ \% also accept $S$. Note that
$1$-sufficient ($1$-necessary) reasons are
just sufficient (necessary) reasons. 
Furthermore, if all features are categorical,
then $\{A_{X_{i_1} \in S_{{i_1},j_1}},
\dots,
A_{X_{i_k} \in S_{{i_k},j_k}}\}$
is a $\delta$-sufficient
reason for a class argument $A_y$ in $\bag_{\rf}$
if and only if every partial assignment
from 
$S_{{i_1},j_1} \times
\dots \times S_{{i_k},j_k}$
is a $\delta$-AXp for $y$ in $\rf$. 
If we have continuous features, this may
not be the case for $\delta \neq 1$ because
the indistinguishability relation does not
necessarily partition the domains of continuous
features into equivalence classes of equal size.

In the next proposition, we use the
following notation:
Given an assignment $\by_F$ to a subsequence of feature random variables $\bY_F$, we let  
$S_{\by_F}$ denote the corresponding set
of feature arguments that contains
the feature argument $A_{X_i \in S^i_j}$ 
if and only if $\by_F$ assigns $U_i = S^i_j$.

\begin{proposition}
\label{prop_MN_queries}
\begin{enumerate}
    \item If all features are categorical, then $Z = |\dom| - |\ambSet(\rf)|  = |\labellings^s(\bag_{\rf})|$ is the number of equivalence classes of non-ambiguous inputs for $\rf$.
    \item Let $\bY_F$ be a subsequence of feature random variables. Then
    $\mathrm{Pl}_{\rf}(C=y, \by_F) = \frac{N_{(C=y, \by_F)}}{Z},$ where
    $N_{(C=y, \by_F)}$ is the number of
    bi-stable labellings that accept all
    arguments in $S_{\by_F} \cup \{A_y\}$.
    \item $\mathrm{P}_{\rf}(C=y \mid \by_F) = \delta$ if and only if
    $S_{\by_F}$ is a $\delta$-sufficient
    reason for $A_y$.
    \item $\mathrm{P}_{\rf}(\by_F \mid C=y) = \delta$ if and only if 
    $S_{\by_F}$ is a $\delta$-necessary
    reason for $A_y$.
\end{enumerate}
\end{proposition}
\if\showproofs1
\begin{proof}
1. Using Proposition \ref{prop_plausibility}, we have that
$$Z = \sum_{\bx} \mathrm{Pl}_{\rf}(\bx)
= \sum_{\bx: L_{\bx} \ \textrm{is bi-stable}} 1 = |\labellings^s(\bag_{\rf})|$$
is the number of bi-stable labellings of the explanation BAG. From this, we obtain the claim by
putting $Z = |\labellings^s(\bag_{\rf})|$ in
Proposition \ref{prop_count_ambigous_vs_labellings}.

2. Using again Proposition \ref{prop_plausibility}, we have
$$P_{\rf}(C=y, \by_F)
= \frac{1}{Z} \sum_{
  \substack{ 
      \bx: L_{\bx} \ \textrm{is bi-stable} \\
      L_{\bx}(S_{(C=y, \by_F)})=\lin
    } 
}    1 
= \frac{N_{(C=y, \by_F)}}{Z}.
$$

3 and 4. Similar to the previous item, we can show that 
$P_{\rf}(C=y) = \frac{N_{(C=y)}}{Z}$
and
$P_{\rf}(\by_F) = \frac{N_{(\by_F)}}{Z}$,
where $N_{(C=y)}$ is the number of
bi-stable labellings that accept $\{A_y\}$
and $N_{(\by_F)}$ is the number of
bi-stable labellings that accept all
arguments in $S_{\by_F}$. Using the previous
item and the definition of conditional probability
we have that 
\begin{align*}
  \mathrm{P}_{\rf}(C=y \mid \by_F)
 &= \frac{\mathrm{P}_{\rf}(C=y, \by_F)}{\mathrm{P}_{\rf}(\by_F)}
 =
\frac{
 \frac{N_{(C=y, \by_F)}}{Z}
}{
 \frac{N_{(\by_F)}}{Z}
}\\
&= 
\frac{N_{(C=y, \by_F)}}{N_{(\by_F)}}
\end{align*}
Note that this is just the percentage of 
labellings that accept both  $A_y$ and $S_{\by_F}$ among those that accept
$S_{\by_F}$. This implies 3. 4 follows
analogously.
\end{proof}
\fi

\subsection{A Probabilistic Approximation Algorithm}

Proposition \ref{prop_ambiguous_complexity}
and the complexity results for deciding
AXps and $\delta$-AXps from \cite{Izza021}
and \cite{WaldchenMHK21} make it unlikely 
that there is an efficient exact algorithm for
computing the partition function and the
probabilities in Proposition \ref{prop_MN_queries}.
We therefore consider a probabilistic
algorithm that approximates the probabilities.
Readers familiar with Bayesian networks may 
notice that $P_{\rf}$ is almost a Bayesian 
network: The variable factors are independent
of all other factors, the tree factors depend
only on the variable factors and the class factor
only on the tree factors. The dependency structure of the factors in $P_{\rf}$ is
therefore acyclic like in a Bayesian network.
However, the class factor cannot be interpreted
as a conditional probability distribution
because it does not define a probability distribution when the configuration of the 
tree factors corresponds to an ambiguous input.
Nevertheless, the acyclic structure allows us
to use forward sampling ideas for Bayesian
networks \cite{koller_probabilistic_2009} to approximate sufficient and
necessary queries.

\begin{figure}
\begin{align*}
    &\textbf{Input: } \textit{ rand. forest } \rf, \textit{ queries } (\bz_1 \mid \by_1), \dots, (\bz_l \mid \by_l) \\[0.05cm]
    &\textbf{Output:} \textit{ estimates for } \mathrm{P}_{\rf}(\bz_1 \mid \by_1), \dots, \mathrm{P}_{\rf}(\bz_l \mid \by_l) \\[0.3cm]
    &\textbf{DO:} \\[0.05cm]
    &|\hspace{0.3cm} E \leftarrow sampleEquivalenceClass(\rf) \\[0.05cm] 
    &|\hspace{0.3cm} \textbf{IF } \outputf_{\rf}(E) \neq \bot: \\[0.05cm] 
    &|\hspace{0.3cm}|\hspace{0.3cm} countNonambiguous() \\[0.05cm] 
    &|\hspace{0.3cm}|\hspace{0.3cm} \bx \leftarrow computeAssignment(E) \\[0.05cm] 
    &|\hspace{0.3cm}|\hspace{0.3cm} \textbf{FOR } i=1 \textbf{ TO } k: \\[0.05cm]
    &|\hspace{0.3cm}|\hspace{0.3cm}|\hspace{0.3cm} \textbf{IF } \textit{$\bx|_{\bY_i} = \by_i$:} \\[0.05cm] 
    &|\hspace{0.3cm}|\hspace{0.3cm}|\hspace{0.3cm}|\hspace{0.3cm} \textbf{IF } \textit{$\bx|_{\bZ_i}=\bz_i$: }
    countPos(\bz_i, \by_i) \\[0.05cm] 
    &|\hspace{0.3cm}|\hspace{0.3cm}|\hspace{0.3cm}|\hspace{0.3cm}\textbf{ELSE: } 
    countNeg(\bz_i, \by_i)  \\[0.05cm] 
    &|\hspace{0.3cm} \textbf{ELSE: } countAmbiguous() \\[0.05cm]
    &\textbf{WHILE } \textit{termination condition not met} \\[0.05cm]
    &\textbf{RETURN } estimates()
\end{align*}
    \caption{Probabilistic approximation algorithm for estimating the percentage of non-ambiguous inputs, and the probabilities of
    sufficient and necessary queries.}
    \label{fig:sampling_algorithm}
\end{figure}
Figure \ref{fig:sampling_algorithm} shows the template of our
algorithm. 
It expects as input a random forest and the 
probabilistic queries that are to be approximated.
The queries consist of sufficient queries (item 2) or necessary queries (item 3)
in Proposition \ref{prop_ambiguous_complexity}. 
The algorithm uses forward sampling (from the
feature variables to the class variable).
It repeatedly samples 
equivalence classes of inputs for $\rf$. Since the tree and class
factors are deterministic, the state of the 
tree and class variables is already determined
by this sample and their state is only computed
if needed. With a slight abuse of notation,
we write $\outputf_{\rf}(E)$ for $\outputf_{\rf}(e)$, where $e \in E$ is an
arbitrary input from the equivalence class
$E$ (recall that all inputs in $E$ are 
indistinguishable for $\rf$).
Ambiguous samples are rejected
immediately, but we keep track of their number
($countAmbiguous()$). For non-ambiguous samples,
we also iterate a counter ($countNonambiguous()$)
and complete the variable assignment. 
The completed samples are used to approximate the queries by relative frequencies. 
More precisely, $countPos(\bz_i, \by_i)$ and $countNeg(\bz_i, \by_i)$
increment counters $N^+_{(\bz_i, \by_i)}$ or $N^-_{(\bz_i, \by_i)}$
that count how often the 
target $\bz_i$ was satisfied or not
when the condition $\by_i$ was satisfied. 
The estimate for the conditional probability
$P_\rf(\bz_i \mid \by_i)$
is
$\frac{N^+_{(\bz_i, \by_i)}}{N^+_{(\bz_i, \by_i)} + N^-_{(\bz_i, \by_i)}}$. The estimate for
the percentage of non-ambiguos input equivalence classes is
$\frac{N_n}{N_n + N_a}$, where $N_a$ ($N_n$)
is the numbers of (non-)ambiguous equivalence
classes that we sampled. Multiplying this
fraction by the number of all equivalence classes
results in an estimate for the number of 
non-ambiguos input equivalence classes,
but we restrict to reporting the percentage as
it is easier to comprehend.
We have the following guarantees, where 
\emph{convergence in probability} means
that the probability that the estimates
deviate from the target by more than an
arbitrarily small $\epsilon$ goes to $0$
as the number of samples goes to $\infty$.
\begin{proposition}
\label{prop_convergence_guarantees}
When sampling inputs uniformly and independently in the algorithm in Figure 2, then $N_n/(N_n + N_a)$ converges 
in probability to the percentage of non-ambiguous equivalence classes and $e_{(\bz_i \mid \by_i)} = N^+_{(\bz_i, \by_i)}/(N^+_{(\bz_i, \by_i)} + N^-_{(\bz_i, \by_i)})$ to $\mathrm{P}_{\rf}(\bz_i \mid \by_i)$
for $1 \leq i \leq l$. Every iteration runs in linear time with respect to
$\rf$ and the number of queries $l$. Furthermore,
if we have $M \geq \frac{3 \ln(2/\delta)}{P(\bz_i \mid \by_i) \cdot \epsilon^2}$ samples for $e_{(\bz_i \mid \by_i)}$, then 
$
P(e_{(\bz_i \mid \by_i)}\in \mathrm{P}_{\rf}(\bz_i \mid \by_i) \cdot (1 \pm \epsilon)) \geq 1 - \delta.
$
\end{proposition}
\if\showproofs1
\begin{proof}
When sampling inputs uniformly and independently, the completion
method will generate uniform and independent inputs for the Markov
network. This is because an assignment to the input variables
uniquely determines the state of the tree and class variables
because the factors are deterministic. 
As our estimates are composed of sums of
independent Bernoulli random variables, whose
individual expected values are the probabilities
of the associated event, the convergence guarantee
in probability for all estimates follows from the law of large numbers.

For the runtime guarantee, note that sampling an input and 
completing it can clearly be done in linear time with respect
to $\rf$. The counting check can be done in linear time with
respect to the number of random variables in $P_{\rf}$, which
is linear with respect to $\rf$. The number of iterations is $l$. 

For the error bound, first recall that
the dependency structure of the Markov network is acyclic as we explained in the main part. This
allows applying forward sampling algorithms for
Bayesian networks. The samples for
$e_{(\bz_i \mid \by_i)}$ can be seen as samples
for the marginal distribution
$P(\bz_i) = P_{\rf}(\bz_i \mid \by_i)$ (rejection sampling).
The error bounds for the queries
therefore follow from the Chernoff bounds for
forward sampling applied to $P(\bz_i)$,
see Section 12.1 in 
\cite{koller_probabilistic_2009}.
\end{proof}
\fi
Let us note that even though every iteration of
our
algorithm runs in
linear time, we may require a large number of iterations until the estimates converge. 
The probabilistic error bound at the end of
Proposition \ref{prop_convergence_guarantees}
shows that
the convergence speed depends on the number of samples generated for
$e_{(\bz_i \mid \by_i)}$. 
If $P_\rf(\by_i)$ is small, 
this will take longer. Typically, the estimates for necessary queries (conditioned on a class label) and sufficient queries for shorter abductive explanations will converge faster.

Formally, we simultaneously approximate
the percentage of non-ambiguous inputs using
Monte-Carlo sampling and the probabilities of
the queries using rejection sampling \cite{koller_probabilistic_2009}. 
Every sample that we create in our algorithm
can be used for the Monte-Carlo approximation,
but only a fraction for individual rejection 
samples. It can be wasteful not to use the
Monte-Carlo samples for the queries. However,
once a sufficiently large number of samples
has been created for the Monte-Carlo 
approximation, we can switch from rejection
sampling to conditional forward sampling. 
That is, if $e_{(\bz_i \mid \by_i)}$ requires additional samples,
we fix the state of the variables $\by_i$ and
sample only the remaining feature variables,
which is justified by the acyclic dependency
structure of the factors in $P_\rf$ that we 
explained at the beginning of this section.

\subsection{Implementation and Experiments}

As a first proof of concept, we implemented
a simple variant of the algorithm in Figure \ref{fig:sampling_algorithm} in Python.
The sourcecode and Jupyter notebooks to reproduce the experiments are available under
\url{https://github.com/nicopotyka/Uncertainpy}
in the folders \emph{src/uncertainpy/explanation}
and \emph{examples/explanations/randomForests}, respectively.
We consider a reason $\by$ \emph{almost sufficient} for $C=y$ if it is $\delta$-sufficient and
$P_\rf(C=y \mid \by) > 1.1 \cdot P_\rf(C=y)$,
that is, $\by$ results in a relative 
increase of the probability of at least
10 \%.  We consider $\by$ \emph{almost necessary} if it is $\delta$-necessary.
In this case, we do not need to take the prior
into account because we sample uniformly from
the partition domains (the prior is therefore always at most $0.5)$. We chose $\delta = 0.9$.

Our implementation works in two stages.
The first stage is analogous to 
Figure \ref{fig:sampling_algorithm} and
the queries are the atomic sufficient and 
necessary queries of the form
$(U_y \mid U_i)$ and $(U_i \mid U_y)$ for
all combinations of feature arguments $U_i$
and class arguments $U_y$. At the end of
stage 1, we report the estimated percentage 
of non-ambiguous inputs and all almost
sufficient and necessary reasons that were
found. We can combine all
almost necessary reasons to a single big
necessary reason for reasons similar to
Proposition \ref{prop_nec_reason_bottom_up}.
However, there may be many more almost
sufficient reasons.
Therefore, in the second stage, the algorithm tries to
find almost sufficient reasons of size $2$.
To this end, for all pairs of features $(U_i, U_j)$,
and all possible assignments $(u_i, u_j)$ of equivalence
classes to these features, we perform
forward sampling conditioned on
the feature assignment $(u_i, u_j)$ and
report the estimate if the probability 
exceeds the $\delta$-threshold.
For every pair $(u_i, u_j)$, the probability
can be estimated quickly. However, since there
can be a large number of pairs, the overall  
runtime can be long and the almost sufficient
reasons of size $2$ are reported continuously
while the sampling procedure is running.

We tested our algorithm on three datasets.
The Iris and PIMA dataset are continuous
datasets that have been considered for counterfactual explanations \cite{WhiteG20}.
In addition, we consider the Mushroom
dataset that contains discrete features.
For the random forest trained on the
Iris dataset, the estimated percentage of
non-ambiguous input equivalence classes is
98 \% and we found several almost sufficient
reasons. These included $1$-sufficient reasons. 
The estimates are based on several hundred examples, 
but since there is uncertainty in the sampling
process, we should be careful and assume that
these are $\delta$-sufficient for $\delta$ close
to $1$, but not necessarily equal to $1$.
$\mathrm{petal length} \in (5.0, 5.14]$ is an example of an almost 
sufficient reason of length $1$ 
for the class $\mathrm{Virginica}$.
The pair $(\mathrm{sepal length} \in (5.45, 5.5],
\mathrm{petal length} \in (2.64, 2.75]
)$ is an almost sufficient reason of length $2$ for
the class $\mathrm{Versicolor}$. We generated 10,000 samples for the first
stage in less than one minute on a Windows
laptop with i7-11800H CPU and 16 GB RAM. 
The second stage produced a variety of 
other sufficient reasons within seconds,
but many redundant reasons are reported in the current version. For the Mushroom dataset,
we found that our random forest learnt that $\mathrm{Odor\_Foul=1}$ is
$0.99$-sufficient for $\mathrm{Poisonous}$
and $\mathrm{Odor\_Foul=0}$ is
$0.98$-necessary for $\mathrm{Edible}$.
We provide more details and examples about the experiments and how to reproduce them in the appendix.

\section{Conclusions and Future Work}

We showed that the decision process of random
forests can be encoded as a bipolar argumentation
problem, which allows finding sufficient and
necessary reasons for classifications using 
argumentation tools. Here, we regarded the
problem as an abstract argumentation problem,
where we only focus on relationships between
arguments, while ignoring their logical structure. The resulting reasoning problems
are often solved using reductions to SAT
\cite{dvorak2012cegartix,beierle2015software,alviano2018pyglaf}, CSP \cite{lagniez2015coquiaas}
or Markov networks \cite{potyka_abstract_2020}.
We used the latter here as it naturally 
leads to almost sufficient and almost
necessary reasons and offers a variety of 
probabilistic approximation algorithms. 
Plans for improving our current work include 
the use of structured argumentation tools that
allow incorporating the logical structure
of trees in the reasoning process \cite{BorgB21,BesnardGHMPST14},
and improving our sampling algorithm by taking
account of less obvious conditional independencies 
in the Markov network structure and applying
rule mining algorithms to replace our current
brute-force approach for finding almost sufficient and almost necessary reasons.





\section{Acknowledgments}
This research was partially funded by the  European Research Council (ERC) under the
European Union’s Horizon 2020 research and innovation programme (grant
agreement No. 101020934, ADIX) and by J.P. Morgan and by the Royal
Academy of Engineering under the Research Chairs and Senior Research
Fellowships scheme.  Any views or opinions expressed herein are solely those of the authors.

\appendix

\section{Proofs}

\setcounter{theorem}{0}
\setcounter{proposition}{0}
\setcounter{lemma}{0}

\begin{proposition}
\begin{itemize}
    \item Deciding if there exists 
    an ambiguous input for a B4L random forest 
    is NP-complete.
    \item Counting the number of ambiguous inputs for a B4L random forest is \#P-complete.
\end{itemize}
\end{proposition}
\begin{proof}
We use the terminology and conventions from \cite{papadimitriou95book} in the proof.

1. Membership follows from observing that an
ambiguous input is a polynomially-sized 
certificate that can be verified in 
polynomial time (because the random forest processes inputs in linear time).
For hardness, we show that there is a 
polynomial-time reduction from 3SAT to 
deciding the existence of ambiguous inputs
for B4L random forests. 
Given an arbitrary
3CNF formula $F: \bigwedge_{i=1}^m \big( \bigvee_{i=1}^3 l_{i,j} \big)$ with $m$ clauses, we create a decision tree of depth 3 for every clause $\big( \bigvee_{i=1}^3 l_{i,j} \big)$. It contains 
three decision nodes that contain the atoms corresponding to $l_{i,1}, l_{i,2}, l_{i,3}$. If $l_{i,j}$ is a positive literal,
the positive branch will go to a leaf labelled $1$ and the
negative branch to $l_{i,j+1}$ if $j<3$ and to a leaf labelled
$0$ otherwise. A negative literal is treated symmetrically by
switching the negative and positive branch. We illustrate the
construction in Figure \ref{fig_clause_2_tree}.
\begin{figure}[tb]
	\centering
		\includegraphics[width=0.35\textwidth]{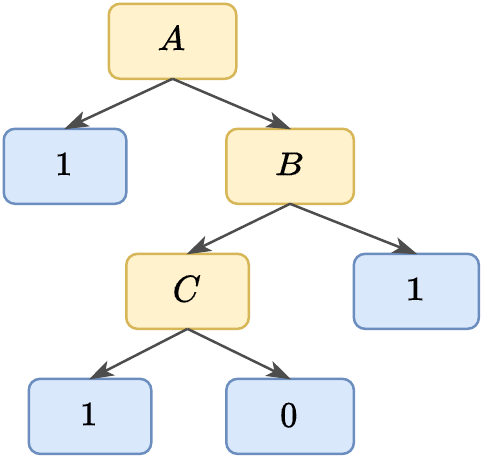}
	\caption{Tree for the 3-clause $A \wedge \neg B \wedge C$.}
	\label{fig_clause_2_tree}
\end{figure}
We now define a B4L random forest consisting of the $m$ trees
obtained from the clauses and $m$ additional tree stumps that
consists of a single leaf that is labelled $0$. Every tree
has constant size and so the size of the random forest is linear
in the size of $F$. Furthermore, 
$F$ is satisfiable 
\begin{itemize}
    \item iff all clauses can be satisfied simultaneously,
    \item iff there is an input such that all clause trees
    output $1$,
    \item iff there is an ambiguos input (because the $m$ 
    clause trees will vote $1$ and the $m$ tree stumps will vote $0$).
\end{itemize}
Hence, $F$ is satisfiable if and only
if the random forest has an ambiguous input. 

2. As (ambiguous) inputs are polynomial in the size of the 
random forest and can be processed in polynomial time, the
problem is in $\#P$. Hardness  follows from observing that
our previous reduction is a parsimonious reduction from
3SAT, that is, the number of satisfying assignments of the
3CNF formula
equals the number of ambiguous inputs of the random forest.
Hence, the claim follows from $\#P$-hardness of $\#3SAT$. 
\end{proof}

\begin{proposition}
$\bag_{\rf}$ can be generated from $\rf$ in quadratic time.
\end{proposition}
\begin{proof}
We have one class argument per class,
one rule argument per leaf in a decision tree and one feature argument per feature condition. Hence, the number of arguments is
linear in the size of the random forest.
Since the number of edges between the arguments
can be at most quadratic, the claim follows.
\end{proof}

\begin{lemma}
Let $L$ be a labelling for $\bag_{\rf}$.
\begin{enumerate}
    \item If $L \in \labellings^c(\bag_{\rf})$, then for all features $X_i$, either
    \begin{itemize}
    \item $L(A_{X_i \in S_{i,j}}) = \lundec$ for all $A_{X_i \in S_{i,j}} \in \args_{X_i}$ or
    \item $L(A_{X_i \in S_{i,j}}) = \lin$ for exactly one $A_{X_i \in S_{i,j}} \in \args_{X_i}$ and $L(A_{X_i \in S_{i,j'}}) = \lout$ for all
    other $A_{X_i \in S_{i,j'}} \in \args_{X_i} \setminus \{A_{X_i \in S_{i,j}}\}$. 
    \end{itemize}
    \item If $L \in \labellings^s(\bag_{\rf})$, then for all features $X_i$,
    $L(A_{X_i \in S_{i,j}}) = \lin$ for exactly one $A_{X_i \in S_{i,j}} \in \args_{X_i}$ and $L(A_{X_i \in S_{i,j'}}) = \lout$ for all
    other $A_{X_i \in S_{i,j'}} \in \args_{X_i} \setminus \{A_{X_i \in S_{i,j}}\}$.
\end{enumerate}
\end{lemma}
\begin{proof}
1. Note that, by construction, all arguments in $\args_{X_i}$
 attack each other.
 Furthermore, feature arguments do not have supporters.
 Hence, if one argument in $\args_{X_i}$ is $\lundec$,
 then all others must be $\lundec$ as well.
 Since arguments in $\args_{X_i}$ are only attacked by arguments in $\args_{X_i}$,
 if one argument in $\args_{X_i}$ is $\lout$, 
 then there must be another argument in $\args_{X_i}$
 that attacks it and is $\lin$. Hence, we are either in
 the situation where all arguments are undecided or 
 at least one argument is $\lin$.
 However, if $L(A_{X_i \in S_{i,j}}) = \lin$, then all other $A_{X_i \in S_{i,j'}} \in \args_{X_i}$ must be out.

2. The claim follows immediately from item 1 because
 bi-stable labellings do not allow labelling arguments
 undecided.
\end{proof}

\begin{lemma}
If $L \in \labellings^s(\bag_{\rf})$, then for all trees $\dt \in \rf$,
$A_{\dt, r} \in \args_{\dt}$ is labelled $\lin$
if and only if $r$ is the active rule in $\dt$ 
for all inputs $\cinput \in S_L$. Furthermore, all other rule
arguments in $\args_{\dt}$ are labelled out.
\end{lemma}
\begin{proof}
Every input
$\cinput \in S_L$ satisfies all
feature constraints accepted by $L$. 
Therefore, a rule argument $A_{\dt, r} \in \args_{\dt}$ 
is supported (attacked) by
an $\lin$-labelled feature argument if and only if
$\cinput$ satisfies (violates) the corresponding feature
condition. Therefore, $A_{\dt, r} \in \args_{\dt}$ is labelled $\lin$ if and only if $r$ is the active rule in $\dt$ for $\cinput$

The second claim follows from the same considerations
because the inactive rules must be attacked by the
$\lin$-labelled arguments.
\end{proof}

\begin{proposition}[Correctness]
If $L \in \labellings^s(\bag_{\rf})$, then for all class arguments
$A_y \in \args_{\class}$,
$L(A_y) = \lin$ if and only if $\outputf_{\rf}(\cinput) = y$ for all $\cinput \in S_L$. Furthermore, if $L(A_y) = \lin$, then
$L(A_{y'}) = \lout$ for all $A_{y'} \in \args_{\class}$.
\end{proposition}
\begin{proof}
If $L(A_y) = \lin$, then all attacking arguments must
be $\lout$ or the supporters must dominate the attackers.
If all attackers are $\lout$, then no rule that picks
another class can be active in $\dt$ and so all active
rules must pick $y$. If the supporters dominate the 
attackers, then a majority of active rules (trees) 
vote for $y$.
Hence, in both cases $\outputf_{\rf}(\cinput) = y$
for all $\cinput \in S_L$.

Conversely, if $\outputf_{\rf}(\cinput) = y$
for all $\cinput \in S_L$,
then a majority of trees vote for $y$. The corresponding
rules will then support $A_y$ and dominate the attackers,
so that $L(A_y) = \lin$.

For the second claim, notice that 
$L(A_y) = \lin$ implies that a majority of $\lin$
rule features support $A_y$. However, by construction,
these rule features also attack all other class arguments.
Therefore, for these arguments, the attackers must
dominate the supporters, so that they will be labelled out.
\end{proof}

\begin{lemma}
\begin{enumerate}
    \item For all $\cinput \in \dom$, $L_{\cinput} \in \labellings^c(\bag_{\rf})$ .
    \item There is at most one $y \in \class$ such that
    $L_{\cinput}(A_y) = \lin$. Furthermore, if $L_{\cinput}(A_y) = \lin$ for some $y \in \class$, then $L_{\cinput}(A_{y'}) = \lout$ for all $y' \in \class \setminus \{y\}$.
\end{enumerate}
\end{lemma}
\begin{proof}
1. We start by checking that the labelling of the 
feature arguments satisfies the definition of a
bi-complete labelling.
For every feature $X_i$, exactly one feature argument
$A_{X_i \in S_{i,j}} \in \args_{X_i}$ is $\lin$ because the feature conditions are mutually exclusive by definition. 
Since $A_{X_i \in S_{i,j}}$ is only attacked by other feature arguments
in $\args_{X_i}$ that are out, it must indeed be $\lin$. 
Since all other feature arguments in $\args_{X_i}$ 
are attacked by $A_{X_i \in S_{i,j}}$ they must indeed be $\lout$. 
Hence, the feature arguments are correctly labelled.

The fact that all rule arguments are correctly labelled
follows by the same arguments that we used in the proof 
of Lemma \ref{lemma_bistable_rule_properties}.

Finally consider a class argument $A_y$. 
Using the same arguments as in the proof of 
Proposition \ref{prop_bistable_class_properties},
we can check that $A_y$ is labelled $\lin$ ($\lout$)
if and only if its supporters (attackers)
dominate its supporters. 
If $L_{\cinput}(A_y) = \lundec$, the number of $\lin$-attackers equals the number of $\lin$-supporters.
The definition of bi-complete labellings could only be
violated if this number is zero. However, since every tree
must contain an active rule, the number must be non-zero.
Hence, $L_{\cinput}$ is a bi-complete labelling.

2. Suppose that $L_{\cinput}(A_y) = \lin$. Then the
number of rule arguments (active rules/ trees) that are
$\lin$ and support $A_y$ must be larger than the number 
of rule arguments that are $\lin$ and attack $A_y$. Hence,
for every other class label $y' \in \class \setminus \{y\}$,
the number of rule arguments that are
$\lin$ and attack $A_{y'}$ must be larger than the number 
of rule arguments that are $\lin$ and support $A_{y'}$.
Hence, $L_{\cinput}(A_{y'}) = \lout$.
\end{proof}

\begin{proposition}[Completeness]
For all inputs $\cinput \in \dom$,
$\outputf_{\rf}(\cinput) \neq \bot$
if and only if
$L_{\cinput} \in \labellings^s(\bag_{\rf})$.
\end{proposition}
\begin{proof}
Item 1 of Lemma \ref{lemma_bistable_input_properties}
guarantees that $L_{\cinput}$ is bi-complete.
By construction of $L_{\cinput}$, 
no feature and
rule arguments are labelled undecided.
Hence, the 
only case in which $L_{\cinput}$ is not a
bi-stable labelling is the case in which 
it labels a class argument undecided. 

If $L_{\cinput}$ is
bi-stable, there must be a class label $y \in \class$ such that the
number of rule arguments (active rules/ trees) that are
$\lin$ and support $A_y$ is larger than the number 
of rule arguments that are $\lin$ and attack $A_y$.
Hence, a majority of trees support $y$ and 
$\outputf_{\rf}(\cinput) = y \neq \bot$.
Conversely, if $\outputf_{\rf}(\cinput) = y \neq \bot$,
a majority of trees support $y$ and the corresponding 
active rule arguments that support $A_y$ will dominate 
$A_y$'s attackers. Hence, $L_{\cinput}(A_y) =\lin$ and
according to item 2 of Lemma \ref{lemma_bistable_input_properties}, $L_{\cinput}(A_{y'}) =\lout$
for the remaining class arguments. Hence, 
$L_{\cinput}$ is bi-stable.
\end{proof}

\begin{proposition}
$|\ambSet(\rf) /\! \equiv_\rf| = |\dom| - |\labellings^s(\bag_{\rf})|$.
\end{proposition}
\begin{proof}
By correctness and completeness of our
encoding, 
$|\labellings^s(\bag_{\rf})|$ is the number
of non-ambiguous equivalence classes, 
that is, $|\ambSet(\rf) /\! \equiv_\rf| + |\labellings^s(\bag_{\rf})| = |\dom|$.
Reordering terms gives the claim.
\end{proof}

\begin{proposition}
If $N \subseteq \args_{\rf}$ is necessary for
$A_y$, then all 
$A \in N$ are necessary for 
$A_y$.
\end{proposition}
\begin{proof}
Consider an arbitrary labelling $L$ that accepts
$A_y$. Then $L$ also accepts $N$ because $N$
is necessary for $A_y$. Since $A \in N$,
$L$ also accepts $A$. Hence, $A$ is also
necessary for accepting $A_y$.
\end{proof}

\begin{proposition}
The explanation plausibility distribution 
$\mathrm{Pl}_{\rf}(\bX)$
can be generated from $\rf$ in log-linear time.
\end{proposition}
\begin{proof}
The proof has been sketched in the main part of
the paper before the proposition.
\end{proof}

\begin{proposition}
For every assignment $\bx$ to $\mathrm{Pl}_{\rf}(\bX)$,
we have $\mathrm{Pl}_{\rf}(\bx) \in \{0,1\}$.
Furthermore, $\mathrm{Pl}_{\rf}(\bx) \neq 0$ if and only if 
$L_{\bx}$ is a bi-stable labelling of the explanation BAG.
\end{proposition}
\begin{proof}
For the first claim, observe that every factor returns either $0$ or $1$.
Since $\mathrm{Pl}_{\rf}(\bx)$ is the product of the returned numbers, it must
be either $0$ or $1$.

For the second claim, first assume that $\mathrm{Pl}_{\rf}(\bx) \neq 0$.
Note that, by definition, $L_{\bx}$ accepts exactly 
one feature argument per feature, one rule argument per tree
and one class argument. Using the same line of reasoning as in item 1 of the proof of Lemma 
\ref{lemma_bistable_input_properties},
we can check that $L_{\bx}$ is bi-stable.

Now assume that $\mathrm{Pl}_{\rf}(\bx) = 0$.
This can only happen if one factor returns $0$.
If a tree factor $\phi_{\dt}(\bX_{\dt})$ returns $0$, 
then $U_{\dt} \neq r$, where $r$ is the active rule
for the assignment of the feature variables. 
Hence, $L_{\bx}(A_{\dt,r}) = \lout$.
However, since $r$ is the active rule, all feature variables
are assigned feature constraints that are consistent with
the rule. Hence, the feature arguments that are in are all
consistent with $A_{\dt,r}$. Hence, all attackers of $A_{\dt,r}$
and must be labelled $\lin$ by a bi-stable labelling.
Hence, $L_{\bx}$ is not a bi-stable labelling.

Similarly, if the class factor $\phi_{\class}(\bX_{\class})$
returns $0$, then the class $y$ assigned to $U_\class$ does 
not win the majority vote. By definition of
$L_{\bx}$, we have $L_{\bx}(A_y) = \lin$, but $A_y$'s supporters
do not dominate its attackers (for otherwise, $y$ would have 
won the majority vote). Hence, $L_{\bx}$ is again not a 
bi-stable labelling.
\end{proof}

\begin{proposition}
\begin{enumerate}
    \item If all features are categorical, then $Z = |\dom| - |\ambSet(\rf)|  = |\labellings^s(\bag_{\rf})|$ is the number of equivalence classes of non-ambiguous inputs for $\rf$.
    \item Let $\bY_F$ be a subsequence of feature random variables. Then
    $\mathrm{Pl}_{\rf}(C=y, \by_F) = \frac{N_{(C=y, \by_F)}}{Z},$ where
    $N_{(C=y, \by_F)}$ is the number of
    bi-stable labellings that accept all
    arguments in $S_{\by_F} \cup \{A_y\}$.
    \item $\mathrm{P}_{\rf}(C=y \mid \by_F) = \delta$ if and only if
    $S_{\by_F}$ is a $\delta$-sufficient
    reason for $A_y$.
    \item $\mathrm{P}_{\rf}(\by_F \mid C=y) = \delta$ if and only if 
    $S_{\by_F}$ is a $\delta$-necessary
    reason for $A_y$.
\end{enumerate}
\end{proposition}
\begin{proof}
1. Using Proposition \ref{prop_plausibility}, we have that
$$Z = \sum_{\bx} \mathrm{Pl}_{\rf}(\bx)
= \sum_{\bx: L_{\bx} \ \textrm{is bi-stable}} 1 = |\labellings^s(\bag_{\rf})|$$
is the number of bi-stable labellings of the explanation BAG. From this, we obtain the claim by
putting $Z = |\labellings^s(\bag_{\rf})|$ in
Proposition \ref{prop_count_ambigous_vs_labellings}.

2. Using again Proposition \ref{prop_plausibility}, we have
$$P_{\rf}(C=y, \by_F)
= \frac{1}{Z} \sum_{
  \substack{ 
      \bx: L_{\bx} \ \textrm{is bi-stable} \\
      L_{\bx}(S_{(C=y, \by_F)})=\lin
    } 
}    1 
= \frac{N_{(C=y, \by_F)}}{Z}.
$$

3 and 4. Similar to the previous item, we can show that 
$P_{\rf}(C=y) = \frac{N_{(C=y)}}{Z}$
and
$P_{\rf}(\by_F) = \frac{N_{(\by_F)}}{Z}$,
where $N_{(C=y)}$ is the number of
bi-stable labellings that accept $\{A_y\}$
and $N_{(\by_F)}$ is the number of
bi-stable labellings that accept all
arguments in $S_{\by_F}$. Using the previous
item and the definition of conditional probability
we have that 
\begin{align*}
  \mathrm{P}_{\rf}(C=y \mid \by_F)
 &= \frac{\mathrm{P}_{\rf}(C=y, \by_F)}{\mathrm{P}_{\rf}(\by_F)}
 =
\frac{
 \frac{N_{(C=y, \by_F)}}{Z}
}{
 \frac{N_{(\by_F)}}{Z}
}\\
&= 
\frac{N_{(C=y, \by_F)}}{N_{(\by_F)}}
\end{align*}
Note that this is just the percentage of 
labellings that accept both  $A_y$ and $S_{\by_F}$ among those that accept
$S_{\by_F}$. This implies 3. 4 follows
analogously.
\end{proof}

\begin{proposition}
When sampling inputs uniformly and independently in the algorithm in Figure 2, then $N_n/(N_n + N_a)$ converges 
in probability to the percentage of non-ambiguous equivalence classes and $e_{(\bz_i \mid \by_i)} = N^+_{(\bz_i, \by_i)}/(N^+_{(\bz_i, \by_i)} + N^-_{(\bz_i, \by_i)})$ to $\mathrm{P}_{\rf}(\bz_i \mid \by_i)$
for $1 \leq i \leq l$. Every iteration runs in linear time with respect to
$\rf$ and the number of queries $l$. Furthermore,
if we have $M \geq \frac{3 \ln(2/\delta)}{P(\bz_i \mid \by_i) \cdot \epsilon^2}$ samples for $e_{(\bz_i \mid \by_i)}$, then we have the error bound
$$
P(e_{(\bz_i \mid \by_i)}\in \mathrm{P}_{\rf}(\bz_i \mid \by_i) \cdot (1 \pm \epsilon)) \geq 1 - \delta.
$$
\end{proposition}
\begin{proof}
When sampling inputs uniformly and independently, the completion
method will generate uniform and independent inputs for the Markov
network. This is because an assignment to the input variables
uniquely determines the state of the tree and class variables
because the factors are deterministic. 
As our estimates are composed of sums of
independent Bernoulli random variables, whose
individual expected values are the probabilities
of the associated event, the convergence guarantee
in probability for all estimates follows from the law of large numbers.

For the runtime guarantee, note that sampling an input and 
completing it can clearly be done in linear time with respect
to $\rf$. The counting check can be done in linear time with
respect to the number of random variables in $P_{\rf}$, which
is linear with respect to $\rf$. The number of iterations is $l$. 

For the error bound, first recall that
the dependency structure of the Markov network is acyclic as we explained in the main part. This
allows applying forward sampling algorithms for
Bayesian networks. The samples for
$e_{(\bz_i \mid \by_i)}$ can be seen as samples
for the marginal distribution
$P(\bz_i) = P_{\rf}(\bz_i \mid \by_i)$ (rejection sampling).
The error bounds for the queries
therefore follow from the Chernoff bounds for
forward sampling applied to $P(\bz_i)$,
see Section 12.1 in 
\cite{koller_probabilistic_2009}.
\end{proof}

\section{Experiments}

We conducted all experiments on a windows
laptop with i7-11800H CPU and 16 GB RAM. 
The sourcecode and Jupyter notebooks to reproduce the experiments are available under
\url{https://github.com/nicopotyka/Uncertainpy}
in the folders \emph{src/uncertainpy/explanation}
and \emph{examples/explanations/randomForests}, respectively.
The experiments for the three different
datasets can be found in three Jupyter
notebooks that work out of the box when
the standard packages
Numpy, Scikit-learn and Pandas are installed.

We used the Iris dataset from Scikit-learn 
(it will run out of the box when Scikit-learn is installed)
and added the Mushroom and PIMA dataset in the source
folder (we used relative paths, so they should also 
work out of the box).

We give an overview of sample size,
runtime, percentage of non-ambiguous inputs
and show some example explanations for each
dataset. In general, generating 10,000 samples
took about 50 seconds on our laptop and scaled
up linearly. For example, generating 100,000
samples took about 500 seconds. 10,000 samples
are usually sufficient to get a good estimate
for the percentage of non-ambiguous inputs.
If the domain is partitioned in many equivalence
classes, we may require more samples to get
a sufficient number of samples for the 
sufficient candidates of length 1 in stage 1
because we currently perform rejection sampling
in this stage. The runtime can probably be further
improved by switching from rejection sampling
to a conditional forward sampling (as we do in stage 2) once sufficient
samples for the estimate of non-ambiguous inputs
have been collected.

\subsection{Iris}

Stage 1 completed in 50 seconds (10,000 samples) and the estimated percentage of non-ambiguous inputs was $98$ \%. We usually had several hundred samples for the estimates of our almost sufficient and necessary reasons of length 1 in 
stage 1. In stage 2, various almost sufficient and necessary reasons have been identified within
seconds (based on about 100 samples) and we stopped
the generation after a few minutes. 
We list the 
probabilities of some hand-picked examples and
refer to the notebook for a more exhaustive list
(the classes are \emph{Setosa}, \emph{Versicolour} and \emph{Virginica}).
\begin{align*}
&P( virginica | 'petal length (cm)'=(5.0, 5.14))=1.0 \\
&(407   \ samples) \notag \\[0.2cm]
&P( virginica | 'petal length (cm)'=(5.14, 5.2))=1.0  \\
&(392   \ samples)  \notag \\[0.2cm]
&P( virginica | 'petal length (cm)'=(5.2, 5.25))=1.0 \\
&(414   \ samples)  \notag \\[0.2cm]
&P( versicolor | 'sepal length (cm)'=(5.45, 5.5), 'petal length (cm)'=(3.19, 4.7))=1.0 \\
&(95  \ samples) \notag  \\[0.2cm]
&P( versicolor | 'sepal width (cm)'=(-\infty, 2.25),  'petal length (cm)'=(2.8, 3.19))=0.94 \\
&(91  \ samples)   \notag 
\end{align*}

\subsection{Mushroom}

Stage 1 completed in 52 seconds (10,000 samples) and the estimated percentage of non-ambiguous inputs was $97.5$ \%. We usually had more than
1,000 samples for the estimates of our almost sufficient and necessary reasons of length 1 in 
stage 1. In stage 2, various almost sufficient and necessary reasons have been identified within
seconds (based on about 100 samples) and we stopped
the generation after a few minutes. 
We list the 
probabilities of some hand-picked examples and
refer to the notebook for a more exhaustive list
(the classes are \emph{edible} and \emph{poisonous}).
Let us note that the notebook shows some reasons
with feature values $-999$. These should be ignored.
When a random forest does not use a feature, our
implementation of the
sampling algorithm currently always return $-999$
for this feature to indicate that the feature is irrelevant. 
\begin{align*}
&P( poisonous | 'odor\_f'=1)=0.99 \\
&(4982  \ samples) \\[0.2cm]
&P( 'odor\_f'=0 | edible)=0.98 \\
&(1380  \ samples)\\[0.2cm]
&P( 'gill-color\_b'=0 | edible)=0.76 \\
&(1380  \ samples)\\[0.2cm]
& P( poisonous | 'cap-shape\_b'=0,  'cap-surface\_f'=0)=0.90 \\
&(98  \ samples)\\[0.2cm]
&P( poisonous | 'cap-shape\_b'=1,  'gill-spacing\_c'=1)=0.90 \\
&(98  \ samples)
\end{align*}

\subsection{PIMA}

Stage 1 completed in 256 seconds (50,000 samples) and the estimated percentage of non-ambiguous inputs was $98$ \%. The number of samples for the estimates of our almost sufficient and necessary reasons of length 1 was between
40 and 260 in 
stage 1. In stage 2, various almost sufficient and necessary reasons have been identified within
seconds (based on about 100 samples) and we stopped
the generation after a few minutes. 
We list the 
probabilities of some hand-picked examples and
refer to the notebook for a more exhaustive list
(the classes are \emph{Pos} and \emph{Neg}).
\begin{align*}
&P( Pos | 'Glucose'=(196.5, \infty))=0.95 \\
&(228   \ samples) \\[0.2cm]
&P( Neg | 'BMI'=(14.94, 15.6))=0.91 \\
&(104   \ samples)\\[0.2cm]
&P( Pos | 'DiabetesPedigreeFunction'=(1.180, 1.181))=0.93 \\
&(44   \ samples)\\[0.2cm]
& P( Neg | 'Pregnancies'=(-\infty, 0.5),  'Glucose'=(42, 47))=0.96 \\
&(99   \ samples)\\[0.2cm]
&P( Neg | 'Pregnancies'=(-\infty, 0.5),  'BMI'=(26.44, 26.5))=0.92 \\
&(99   \ samples)
\end{align*}


\begin{thebibliography}{}

\bibitem[ACLSL08]{amgoud2008bipolarity}
Leila Amgoud, Claudette Cayrol, Marie-Christine Lagasquie-Schiex, and Pierre
  Livet.
\newblock On bipolarity in argumentation frameworks.
\newblock {\em International Journal of Intelligent Systems},
  23(10):1062--1093, 2008.

\bibitem[Alv18]{alviano2018pyglaf}
Mario Alviano.
\newblock The pyglaf argumentation reasoner.
\newblock In {\em International Conference on Logic Programming (ICLP)}.
  Schloss Dagstuhl-Leibniz-Zentrum fuer Informatik, 2018.

\bibitem[BB21]{BorgB21}
AnneMarie Borg and Floris Bex.
\newblock Necessary and sufficient explanations for argumentation-based
  conclusions.
\newblock In Jirina Vejnarov{\'{a}} and Nic Wilson, editors, {\em European
  Conference on Symbolic and Quantitative Approaches to Reasoning with
  Uncertainty {(ECSQARU)}}, volume 12897 of {\em LNCS}, pages 45--58. Springer,
  2021.

\bibitem[BBP15]{beierle2015software}
Christoph Beierle, Florian Brons, and Nico Potyka.
\newblock A software system using a sat solver for reasoning under complete,
  stable, preferred, and grounded argumentation semantics.
\newblock In {\em Joint German/Austrian Conference on Artificial Intelligence
  (KI)}, pages 241--248. Springer, 2015.

\bibitem[BGH{\etalchar{+}}14]{BesnardGHMPST14}
Philippe Besnard, Alejandro~Javier Garc{\'{\i}}a, Anthony Hunter, Sanjay
  Modgil, Henry Prakken, Guillermo~Ricardo Simari, and Francesca Toni.
\newblock Introduction to structured argumentation.
\newblock {\em Argument Comput.}, 5(1):1--4, 2014.

\bibitem[BGvdTV10]{boella2010support}
Guido Boella, Dov~M Gabbay, Leon van~der Torre, and Serena Villata.
\newblock Support in abstract argumentation.
\newblock In {\em International Conference on Computational Models of Argument
  (COMMA)}, pages 40--51. Frontiers in Artificial Intelligence and
  Applications, IOS Press, 2010.

\bibitem[Bre01]{breiman2001random}
Leo Breiman.
\newblock Random forests.
\newblock {\em Machine learning}, 45(1):5--32, 2001.

\bibitem[CLS13]{cayrol_bipolarity_2013}
Claudette Cayrol and Marie-Christine Lagasquie-Schiex.
\newblock Bipolarity in argumentation graphs: {Towards} a better understanding.
\newblock {\em International Journal of Approximate Reasoning}, 54(7):876--899,
  2013.
\newblock Publisher: Elsevier.

\bibitem[CRA{\etalchar{+}}21]{Cyras0ABT21}
Kristijonas Cyras, Antonio Rago, Emanuele Albini, Pietro Baroni, and Francesca
  Toni.
\newblock Argumentative {XAI:} {A} survey.
\newblock In Zhi{-}Hua Zhou, editor, {\em International Joint Conference on
  Artificial Intelligence {(IJCAI)}}, pages 4392--4399. ijcai.org, 2021.

\bibitem[DJWW12]{dvorak2012cegartix}
Wolfgang Dvor{\'a}k, Matti J{\"a}rvisalo, Johannes~Peter Wallner, and Stefan
  Woltran.
\newblock Cegartix: A sat-based argumentation system.
\newblock In {\em Pragmatics of SAT Workshop (POS)}, 2012.

\bibitem[Dun95]{dung_acceptability_1995}
Phan~Minh Dung.
\newblock On the acceptability of arguments and its fundamental role in
  nonmonotonic reasoning, logic programming and n-person games.
\newblock {\em Artificial intelligence}, 77(2):321--357, 1995.
\newblock Publisher: Elsevier.

\bibitem[HII{\etalchar{+}}22]{HuangIICA022}
Xuanxiang Huang, Yacine Izza, Alexey Ignatiev, Martin~C. Cooper, Nicholas
  Asher, and Jo{\~{a}}o Marques{-}Silva.
\newblock Tractable explanations for d-dnnf classifiers.
\newblock In {\em {AAAI} Conference on Artificial Intelligence ({AAAI})}, pages
  5719--5728. {AAAI} Press, 2022.

\bibitem[IISM22]{IgnatievIS022}
Alexey Ignatiev, Yacine Izza, Peter~J. Stuckey, and Jo{\~{a}}o Marques{-}Silva.
\newblock Using maxsat for efficient explanations of tree ensembles.
\newblock In {\em {AAAI} Conference on Artificial Intelligence ({AAAI})}, pages
  3776--3785. {AAAI} Press, 2022.

\bibitem[IM21]{Izza021}
Yacine Izza and Jo{\~{a}}o Marques{-}Silva.
\newblock On explaining random forests with {SAT}.
\newblock In Zhi{-}Hua Zhou, editor, {\em International Joint Conference on
  Artificial Intelligence, {(IJCAI)}}, pages 2584--2591, 2021.

\bibitem[KF09]{koller_probabilistic_2009}
Daphne Koller and Nir Friedman.
\newblock {\em Probabilistic graphical models: principles and techniques}.
\newblock MIT press, 2009.

\bibitem[LL17]{lundberg2017unified}
Scott~M Lundberg and Su-In Lee.
\newblock A unified approach to interpreting model predictions.
\newblock In {\em International Conference on Neural Information Processing
  Systems (NeurIPS)}, pages 4768--4777, 2017.

\bibitem[LLM15]{lagniez2015coquiaas}
Jean-Marie Lagniez, Emmanuel Lonca, and Jean-Guy Mailly.
\newblock Coquiaas: A constraint-based quick abstract argumentation solver.
\newblock In {\em International Conference on Tools with Artificial
  Intelligence (ICTAI)}, pages 928--935. IEEE, 2015.

\bibitem[MI22]{Silva22}
Jo{\~{a}}o Marques{-}Silva and Alexey Ignatiev.
\newblock Delivering trustworthy {AI} through formal {XAI}.
\newblock In {\em {AAAI} Conference on Artificial Intelligence {(AAAI)}}, 2022.

\bibitem[ON08]{orenN08}
Nir Oren and Timothy~J. Norman.
\newblock Semantics for evidence-based argumentation.
\newblock In {\em International Conference on Computational Models of Argument
  (COMMA)}, pages 276--284. {IOS} Press, 2008.

\bibitem[Pap95]{papadimitriou95book}
Christos~H. Papadimitriou.
\newblock {\em Computational Complexity}.
\newblock Addison-Wesley, 1995.

\bibitem[PMT18]{plumb2018model}
Gregory Plumb, Denali Molitor, and Ameet Talwalkar.
\newblock Model agnostic supervised local explanations.
\newblock In {\em International Conference on Neural Information Processing
  Systems (NeurIPS)}, pages 2520--2529, 2018.

\bibitem[Pot20]{potyka_abstract_2020}
Nico Potyka.
\newblock Abstract {Argumentation} with {Markov} {Networks}.
\newblock In {\em European {Conference} on {Artificial} {Intelligence}
  ({ECAI})}, pages 865--872, 2020.

\bibitem[Pot21]{potyka_generalizing_2021}
Nico Potyka.
\newblock Generalizing {Complete} {Semantics} to {Bipolar} {Argumentation}
  {Frameworks}.
\newblock In {\em European {Conference} on {Symbolic} and {Quantitative}
  {Approaches} to {Reasoning} with {Uncertainty} ({ECSQARU} 2021)}, Lecture
  {Notes} in {Computer} {Science}, pages 130--143. Springer, 2021.

\bibitem[RSG16]{ribeiro2016should}
Marco~Tulio Ribeiro, Sameer Singh, and Carlos Guestrin.
\newblock " why should i trust you?" explaining the predictions of any
  classifier.
\newblock In {\em ACM SIGKDD international conference on knowledge discovery
  and data mining}, pages 1135--1144, 2016.

\bibitem[SCD18]{ShihCD18}
Andy Shih, Arthur Choi, and Adnan Darwiche.
\newblock A symbolic approach to explaining bayesian network classifiers.
\newblock In J{\'{e}}r{\^{o}}me Lang, editor, {\em International Joint
  Conference on Artificial Intelligence, {IJCAI}}, pages 5103--5111. ijcai.org,
  2018.

\bibitem[VBP21]{vassiliades2021argumentation}
Alexandros Vassiliades, Nick Bassiliades, and Theodore Patkos.
\newblock Argumentation and explainable artificial intelligence: a survey.
\newblock {\em The Knowledge Engineering Review}, 36, 2021.

\bibitem[WdG20]{WhiteG20}
Adam White and Artur~S. d'Avila Garcez.
\newblock Measurable counterfactual local explanations for any classifier.
\newblock In {\em European Conference on Artificial Intelligence {(ECAI)})},
  2020.

\bibitem[WMHK21]{WaldchenMHK21}
Stephan W{\"{a}}ldchen, Jan MacDonald, Sascha Hauch, and Gitta Kutyniok.
\newblock The computational complexity of understanding binary classifier
  decisions.
\newblock {\em J. Artif. Intell. Res.}, 70:351--387, 2021.

\bibitem[WMR17]{wachter2017counterfactual}
Sandra Wachter, Brent Mittelstadt, and Chris Russell.
\newblock Counterfactual explanations without opening the black box: Automated
  decisions and the gdpr.
\newblock {\em Harv. JL \& Tech.}, 31:841, 2017.

\end{thebibliography}
\bibliographystyle{apalike}
\newcommand{\etalchar}[1]{$^{#1}$}

\end{document}